\documentclass[sigconf]{acmart}

\usepackage{bm}
\usepackage{booktabs}
\usepackage{balance}
\usepackage{graphicx}
\usepackage{caption}
\usepackage{subfigure}
\usepackage{amssymb,amsfonts}
\usepackage{algorithmic}
\usepackage[linesnumbered,ruled]{algorithm2e}
\usepackage{textcomp}
\usepackage{xcolor}
\usepackage{amsmath}
\usepackage{mathtools}
\usepackage{multirow}

\def\BibTeX{{\rm B\kern-.05em{\sc i\kern-.025em b}\kern-.08emT\kern-.1667em\lower.7ex\hbox{E}\kern-.125emX}}
    
\setcopyright{acmlicensed} 
\begin{document}

\fancyhead{}

\copyrightyear{2019}
\acmYear{2019}
\acmConference[CIKM '19]{The 28th ACM International Conference on Information and
Knowledge Management}{November 3--7, 2019}{Beijing, China}
% \acmBooktitle{The 28th ACM International Conference on Information and Knowledge
% Management (CIKM '19), November 3--7, 2019, Beijing, China}
\acmPrice{15.00}
\acmDOI{12.3456}
\acmISBN{123-4-5678-9102-3/19/11}

\title{Privacy-Preserving Tensor Factorization for Collaborative Health Data Analysis}

\newcommand{\methodName}{DPFact}

\author{Jing Ma}
\affiliation{%
  \institution{Emory University}
}
\email{jing.ma@emory.edu}

\author{Qiuchen Zhang}
\affiliation{%
  \institution{Emory University}
}
\email{qzhan84@emory.edu}

\author{Jian Lou}
\affiliation{%
  \institution{Emory University}
}
\email{jian.lou@emory.edu}

\author{Joyce. C. Ho}
\affiliation{%
  \institution{Emory University}
}
\email{joyce.c.ho@emory.edu}

\author{Li Xiong}
\affiliation{%
  \institution{Emory University}
}
\email{lxiong@emory.edu}

\author{Xiaoqian Jiang}
\affiliation{%
  \institution{UT Health Science Center at Houston}
}
\email{xiaoqian.jiang@uth.tmc.edu}

\begin{abstract}
Tensor factorization has been demonstrated as an efficient approach for computational phenotyping, where massive electronic health records (EHRs) are converted to concise and meaningful clinical concepts. While distributing the tensor factorization tasks to local sites can avoid direct data sharing, it still requires the exchange of intermediary results which could reveal sensitive patient information. 
Therefore, the challenge is how to jointly decompose the tensor under rigorous and principled privacy constraints, while still support the model's interpretability.

We propose \methodName, a privacy-preserving collaborative tensor factorization method for computational phenotyping using EHR.
It embeds advanced privacy-preserving mechanisms with collaborative learning.
Hospitals can keep their EHR database private but also collaboratively learn meaningful clinical concepts by sharing differentially private intermediary results.
Moreover, \methodName~solves the heterogeneous patient population using a structured sparsity term.
In our framework, each hospital decomposes its local tensors, and sends the updated intermediary results with output perturbation every several iterations to a semi-trusted server which generates the phenotypes.
The evaluation on both real-world and synthetic datasets demonstrated that under strict privacy constraints, our method is more accurate and communication-efficient than state-of-the-art baseline methods.
\end{abstract}

\begin{CCSXML}
<ccs2012>
<concept>
<concept_id>10002978.10002991.10002995</concept_id>
<concept_desc>Security and privacy~Privacy-preserving protocols</concept_desc>
<concept_significance>500</concept_significance>
</concept>
<concept>
<concept_id>10010147.10010257.10010293.10010309</concept_id>
<concept_desc>Computing methodologies~Factorization methods</concept_desc>
<concept_significance>500</concept_significance>
</concept>
<concept>
<concept_id>10010405.10010444.10010449</concept_id>
<concept_desc>Applied computing~Health informatics</concept_desc>
<concept_significance>300</concept_significance>
</concept>
</ccs2012>
\end{CCSXML}

\ccsdesc[500]{Security and privacy~Privacy-preserving protocols}
\ccsdesc[500]{Computing methodologies~Factorization methods}
\ccsdesc[300]{Applied computing~Health informatics}

\keywords{Phenotyping; Tensor Factorization; Collaborative Learning; Differential Privacy}

\maketitle

\section{Introduction}
Electronic Health Records (EHRs) have become an important source of comprehensive information for patients' clinical histories.
While EHR data can help advance biomedical discovery, this requires an efficient conversion of the data to succinct and meaningful patient characterizations.
Computational phenotyping is the process of transforming the noisy, massive EHR data into meaningful medical concepts that can be used to predict the risk of disease for an individual, or the response to drug therapy.
Phenotyping can be used to assist precision medicine, speedup biomedical discovery, and improve healthcare quality \cite{wei2015extracting, richesson2016clinical}. 

Yet, extracting precise and meaningful phenotypes from EHRs is challenging because observations in EHRs are high-dimensional and heterogeneous,
which leads to poor interpretability and research quality for scientists \cite{wei2015extracting}. Traditional phenotyping approaches require the involvement of medical domain experts, which is time-consuming and labor-intensive. Recently, unsupervised learning methods have been demonstrated as a more efficient approach for computational phenotyping. Although these methods do not require experts to manually label the data, they require large volumes of EHR data.
A popular unsupervised phenotyping approach is tensor factorization  \cite{kim2017federated, wang2015rubik,ho2014marble}.
Not only can tensors capture the interactions between multiple sources (e.g, specific procedures that are used to treat a disease), it can identify patient subgroups and extract concise and potentially more interpretable results by utilizing the multi-way structure of a tensor. 

However, one existing barrier for high-throughput tensor factorization is that EHRs are fragmented and distributed among independent medical institutions, where healthcare practises are different due to heterogeneous patients populations. One of the reasons is that different hospitals or medical sites differ in the way they manage patients \cite{xu2016hospital}. Moreover, effective phenotyping requires a large amount of data to guarantee its reliance and generalizability. Simply analyzing data from single source leads to poor accuracy and bias, which would reduce the quality and efficiency of patients' care. 

Recent studies have suggested that the integration of health records can provide more benefits \cite{greenhalgh2010adoption}, which motivated the application of federated tensor learning framework \cite{kim2017federated}. It can mitigate privacy issues under the distributed data setting while achieves high global accuracy and data harmonization via federated computation.
But this method has inherent limitations of federated learning: 1) high communication cost; 2) reduced accuracy due to local non-IID data (i.e., patient heterogeneity); and 3) no formal privacy guarantee of the intermediary results shared between local sites and the server, which makes patient data at risk of leakage.

In this paper, we propose \methodName, a differentially private collaborative tensor factorization framework based on Elastic Averaging Stochastic Gradient Descent (EASGD) for computational phenotyping. \methodName~assumes all sites share a common model which can be learnt jointly from each site through communication with a central parameter server.
Each site performs its own tensor factorization task to discover both common and distinct latent components, 
while benefiting from the intermediary results generated by other sites. The intermediary results uploaded still contain sensitive information about the patients. Several studies have shown that machine learning models can be used to extract sensitive information used in the input training data through membership inference attacks or model inversion attacks both in the centralized setting \cite{shokri2017membership, fredrikson2015model} and federated setting \cite{hitaj2017deep}. Since we assume the central server and participants are honest-but-curious, hence a formal differential privacy guarantee is desired.  \methodName~tackles the privacy issue with a well-designed data-sharing strategy, combined with the rigorous zero-concentrated differential privacy (zCDP) technique \cite{dwork2016concentrated, yu2019differentially} which is a strictly stronger definition
than $(\epsilon, \delta)$-differential privacy considered as the dominant standard for strong privacy protection \cite{dwork2010boosting, dwork2014algorithmic, dwork2016concentrated}. We briefly summarize our contributions as:

\textbf{1) Efficiency.} \methodName~achieves higher accuracy and faster convergence rate than the state-of-the-art federated learning method. It also beats the federated learning method in achieving lower communication cost thanks to the elimination of auxiliary parameters (e.g., in the ADMM approach) and allows each local site to perform most of the computation.

\textbf{2) Utility.} \methodName~supports phenotype discovery even with a rigorous privacy guarantee. By incorporating a $l_{2,1}$ regularization term, \methodName~can jointly decompose local tensors with different distribution patterns and discover both the globally shared and the distinct, site-specific phenotypes.

\textbf{3) Privacy.} \methodName~is a privacy-preserving collaborative tensor factorization framework. By applying zCDP mechanisms, it guarantees that there is no inadvertent patient information leakage in the process of intermediary results exchange with high probability which is quantified by privacy parameters. 

We evaluate \methodName~on two publicly-available large EHR datasets and a synthetic dataset. The performance of \methodName~is assessed from the following three aspects including efficiency measured by accuracy and communication cost, utility measured by phenotype discovery ability and the evaluation on the effect of privacy.

\section{Preliminaries and Notations}
This section describes the preliminaries used in this paper, including tensor factorization,  $(\epsilon, \delta)$-differential privacy, and zCDP.

\begin{table}
\setlength{\abovecaptionskip}{-0.25cm}
\setlength{\belowcaptionskip}{0cm}
  \centering
  \begin{tabular}{c l}
  \toprule
    \textbf{Symbols}&\textbf{Descriptions}\\
    \midrule
    $\otimes$ & Kronecker product\\
    $\odot$ & Khatri-Rao product\\
    $\circ$ & Outer Product\\
    $\ast$ & Element-wise Product\\
    $N$ & Number of modes\\
    $T$ & Number of local sites\\
    $R$ & Number of ranks\\
    $X_{(n)}$ & $n$-mode matricization of tensor $\mathcal{O}$\\
    $\mathcal{X}, \textbf{X}, \textbf{x}$ & Tensor, matrix, vector\\
    $\widehat{\textbf{B}}, \widehat{\textbf{C}}$ & Global factor matrices \\
    $\textbf{A}^{[t]}, \textbf{B}^{[t]}, \textbf{C}^{[t]}$ & Local factor matrices at the $t$-th site\\
    $\mathcal{X}^{[t]}$ & Local tensor at the $t$-th site\\
    $\textbf{x}_{i:}, \textbf{x}_{:r}$ & Row vector, Column vector\\
    \bottomrule
\end{tabular}
\caption{Symbols and Notations}
\label{tab:freq}
\end{table}

\subsection{Tensor Factorization}

\begin{definition}
(\textit{Khatri-Rao product}). Khatri-Rao product is the ``columnwise" Kronecker product of two matrices $\mathbf{A}\in \mathbb{R}^{I\times R}$ and $\mathbf{B}\in \mathbb{R}^{J\times R}$. The result is a matrix of size $(IJ\times R)$ and defined by
$$
\mathbf{A}\odot \mathbf{B}=\left[\mathbf{a}_1\otimes \mathbf{b}_1\cdots \mathbf{a}_R\otimes \mathbf{b}_R\right]
$$
Here, $\otimes$ denotes the \textit{Kronecker product}. The \textit{Kronecker product} of two vectors $\mathbf{a}\in \mathbb{R}^{I}$, $\mathbf{b}\in \mathbb{R}^{J}$ is 
$$\mathbf{a}\otimes \mathbf{b}=\left[\begin{array}{c}
a_1 \mathbf{b}\\
\vdots\\
a_I \mathbf{b}
\end{array}\right]$$
\end{definition} 

\begin{definition}
(\textit{CANDECOMP-PARAFAC Decomposition}). 
% \begin{figure}[htbp]
% \centering{\includegraphics[width=3in]{figures/cp.png}}
% \caption{CP decomposition of a three-way tensor into $R$ rank-one tensors}
% \label{fig}
% \end{figure}
The CANDECOMP-PARAFAC (CP) decomposition is to approximate the original tensor $\mathcal{O}$ by the sum of $R$ rank-one tensors. $R$ is the rank of tensor $\mathcal{O}$, It can be expressed as
\begin{equation}
\mathcal{O} \approx \mathcal{X}=\sum\limits_{r=1}^{R}\textbf{a}_{:r}^{(1)}\circ \cdots \circ \textbf{a}_{:r}^{(N)},\label{eq:cp}
\end{equation}
where $\textbf{a}_{:r}^{(n)}$ represents the $r^{th}$ column of $A^{(n)}$ for $n=1,\cdots,N$ and $r=1,\cdots,R$. $A^{(n)}$ is the $n$-mode factor matrix consisting of $R$ columns representing $R$ latent components which can be represented as
\begin{equation*}
A^{(n)}=\left[ \textbf{a}_{:1}^{(n)}\cdots \textbf{a}_{:R}^{(n)} \right],
\end{equation*}
so that $A^{(n)}$ is of size $I_n \times R$ for $n=1,\cdots ,N$, and the equation of \eqref{eq:cp} can also be represented as
\begin{equation}
[\![A^{(1)},\cdots,A^{(N)}]\!]=\sum\limits_{r=1}^{R}\textbf{a}_{:r}^{(1)}\circ \cdots \circ \textbf{a}_{:r}^{(N)}.
\end{equation}
Note that in this formulation, the scalar weights for each rank-one tensor are assumed to be absorbed into the factors.

In the way of a three-mode tensor $\mathcal{O}\in \mathbb{R}^{I\times J\times K}$, the CP decomposition can be represented as
\begin{equation}
\mathcal{O} \approx \mathcal{X}=\sum\limits_{r=1}^{R}\textbf{a}_{:r}\circ \textbf{b}_{:r}\circ \textbf{c}_{:r},
\end{equation}
where $\textbf{a}_{:r}\in \mathbb{R}^I$, $\textbf{b}_{:r}\in \mathbb{R}^J$, $\textbf{c}_{:r}\in \mathbb{R}^K$ are the $r$-th column vectors within the three factor matrices $\textbf{A}\in \mathbb{R}^{I\times R}$, $\textbf{B}\in \mathbb{R}^{J\times R}$, $\textbf{C}\in \mathbb{R}^{K\times R}$.
\end{definition}

\subsection{Differential Privacy}

Differential privacy \cite{ dwork2014algorithmic, dwork2016concentrated} has been demonstrated as a strong standard to provide privacy guarantees for algorithms on aggregate database analysis, which in our case is a collaborative tensor factorization algorithm analyzing distributed tensors with differential privacy.  

\begin{definition}
(\textit{($\epsilon$-$\delta$)-Differential Privacy}) \cite{dwork2014algorithmic}. Let $\mathcal{D}$ and $\mathcal{D}'$ be two neighboring datasets that differ in at most one entry. A randomized algorithm $\mathcal{A}$ is ($\epsilon$-$\delta$)-differentially private if for all $\mathcal{S}\subseteq$ Range$(\mathcal{A})$:
$$
Pr\left[\mathcal{A}(\mathcal{D})\in \mathcal{S}\right] \leq e^{\epsilon} Pr\left[\mathcal{A}(\mathcal{D'})\in \mathcal{S}\right]+\delta,
$$
where $\mathcal{A}(\mathcal{D})$ represents the output of $\mathcal{A}$ with an input of $\mathcal{D}$. 
\end{definition}

The above definition suggests that with a small $\epsilon$, an adversary almost cannot distinguish the outputs of an algorithm with two neighboring datasets $\mathcal{D}$ and $\mathcal{D}'$ as its inputs. While $\delta$ allows a small probability of failing to provide this guarantee. Differential privacy is defined using a pair of neighboring databases which in our work are two tensors and differ in only one entry. 

\begin{definition}
(\textit{$L_2$-sensitivity}) \cite{dwork2014algorithmic}. For two neighboring datasets $\mathcal{D}$ and $\mathcal{D}'$ differing in at most one entry, the $L_2$-sensitivity of an algorithm $\mathcal{A}$ is the maximum change in the $l_2$-norm of the output value of algorithm $\mathcal{A}$ regarding the two neighboring datasets:
$$
\Delta_2(\mathcal{A}) = \sup_{\mathcal{D},\mathcal{D'}}{\lVert \mathcal{A}(\mathcal{D})-\mathcal{A}(\mathcal{D'}) \rVert}_2.
$$
\end{definition}

\begin{theorem} \label{theorem:gaussian}
(\textit{(Gaussian Mechanism)}) \cite{dwork2014algorithmic}. Let $\epsilon\in(0,1)$ be arbitrary. For $c^2 > 2\ln(1.25/\delta)$, the Gaussian Mechanism with parameter $\sigma \geq c\Delta_2(\mathcal{A})/\epsilon$,  adding noise scaled to $\mathcal{N}(0, \sigma^2)$ to each component of the output of algorithm $\mathcal{A}$, is $(\epsilon$-$\delta)$-differentially private.
\end{theorem}

\subsection{Concentrated Differential Privacy}

Concentrated differential privacy (CDP) is introduced by Dwork and Rothblum \cite{dwork2016concentrated} as a generalization of differential privacy which provides sharper analysis of many privacy-preserving computations. Bun and Steinke \cite{bun2016concentrated} propose an alternative formulation of CDP called "zero-concentrated differential privacy" (zCDP) which utilizes the R\'enyi divergence between probability distributions to measure the requirement of the privacy loss random variable to be sub-gaussian and provides tighter privacy analysis.

\begin{definition} (\textit{Zero-Concentrated Differential Privacy (zCDP) \cite{bun2016concentrated}}) A randomized mechanism $\mathcal{A}$ is $\rho$-zero concentrated differentially private if for any two neighboring databases $\mathcal{D}$ and $\mathcal{D}'$ differing in at most one entry and all $\alpha \in(1, \infty)$, 

$$
D_{\alpha}\left(\mathcal{A}(\mathcal{D}) \| \mathcal{A}\left(\mathcal{D}^{\prime}\right)\right) \triangleq \frac{1}{\alpha-1} \log \left(\mathbb{E}\left[e^{(\alpha-1) L^{(o)}}\right]\right) \leq \rho \alpha,
$$
where $D_{\alpha}\left(\mathcal{A}(\mathcal{D}) \| \mathcal{A}\left(\mathcal{D}^{\prime}\right)\right)$ is called $\alpha$-R\'enyi divergence between the distributions of $\mathcal{A}(\mathcal{D})$ and $\mathcal{A}\left(\mathcal{D}^{\prime}\right)$, and $L(o)$ is the privacy loss random variable which is defined as:

$$
L_{\left(\mathcal{A}(\mathcal{D})| | \mathcal{A}\left(\mathcal{D}^{\prime}\right)\right)}^{(o)} \triangleq \log \frac{\operatorname{Pr}(\mathcal{A}(\mathcal{D})=o)}{\operatorname{Pr}\left(\mathcal{A}\left(\mathcal{D}^{\prime}\right)=o\right)}.
$$

\end{definition}

The following propositions of zCDP will be used in this paper.

\begin{proposition} \cite{bun2016concentrated} The Gaussian mechanism with noise $\mathcal{N}(0, \sigma^2)$ where $\sigma=\sqrt{1 /(2 \rho)} \Delta_{2}$ satisfies $\rho$-zCDP.
\end{proposition}

\begin{proposition} \cite{bun2016concentrated} If a randomized mechanism $\mathcal{A}$ is $\rho$-CDP, then $\mathcal{A}$ is ($\epsilon^{\prime}$, $\delta$)-DP for any $\delta$ with $\epsilon^{\prime}$ = $\rho+\sqrt{4 \rho \log (1 / \delta)}$; For $\mathcal{A}$ to satisfy $(\epsilon,\delta)$-DP, it suffices to satisfy $\rho$-zCDP by setting $\rho \approx \frac{\epsilon^{2}}{4 \log (1 / \delta)}$.
\end{proposition}

\begin{proposition}
(\textit{(Serial composition \cite{bun2016concentrated})})
Let $\mathcal{A} : \mathcal{D}^{n} \rightarrow \mathcal{Y}$ and $\mathcal{A}^{\prime} : \mathcal{D}^{n} \rightarrow \mathcal{Z}$ be randomized algorithms. Suppose $\mathcal{A}$ is $\rho$-zCDP and $\mathcal{A}^{\prime}$ is $\rho^{\prime}$-zCDP. Define $\mathcal{A}^{\prime \prime} : \mathcal{D}^{n} \rightarrow \mathcal{Y} \times \mathcal{Z}$ by $\mathcal{A}^{\prime \prime}=\left(\mathcal{A}, \mathcal{A}^{\prime}\right)$. Then $\mathcal{A}^{\prime \prime}$ is $(\rho+\rho^{\prime})$-zCDP.
\end{proposition}

\begin{proposition}
(\textit{(Parallel composition \cite{yu2019differentially})})
Suppose that a mechanism $\mathcal{A}$ consists of a sequence of $T$ adaptive mechanisms, $\mathcal{A}_{1}, \dots, \mathcal{A}_{T}$, where each $\mathcal{A}_{t} : \Pi_{j=1}^{i t e r-1} \mathcal{O}_{j} \times \mathcal{D}_{t} \rightarrow \mathcal{O}_{i t e r}$ and $\mathcal{A}_{t}$ satisfies $\rho_{t}$-zCDP. Let $\mathcal{D}_{1}, \ldots, \mathcal{D}_{T}$ be a randomized partition of the input $\mathcal{D}$. The mechanism $\mathcal{A}(\mathcal{D})=\left(\mathcal{A}_{1}\left(\mathcal{D}_{1}\right), \ldots, \mathcal{A}_{T}\left(\mathcal{D}_{T}\right)\right)$ satisfies $\frac{1}{T} \sum_{t=1}^{T} \rho_{t}$-zCDP.
\end{proposition}

\section{\methodName}
In this section, we first provide a general overview and then present detailed formulation of the optimization problem.
 
\subsection{Overview}
\methodName~is a distributed tensor factorization model that preserves differential privacy.
Our goal is to learn computational phenotypes from horizontally partitioned patient data (e.g., each hospital has its own patient data with the same medical features). Since we assume the central server and participants are honest-but-curious which means they will not deviate from the prescribed protocol but they are curious about others secrets and try to find out as much as possible about them. Therefore the patient data cannot be collected at a centralized location to construct a global tensor $\mathcal{O}$.
Instead, we assume that there are $T$ local sites and a central server that communicates the intermediary results between the local sites.
Each site performs tensor factorization on the local data and shares privacy-preserving intermediary results with the centralized server (Figure \ref{fig:algorithm}).

\begin{figure}[htbp]
\setlength{\abovecaptionskip}{0cm}
\setlength{\belowcaptionskip}{-0.25cm}
\centering
\includegraphics[width=0.99\linewidth, trim={0.5mm 0.5mm 0.5mm 0.5mm},clip]{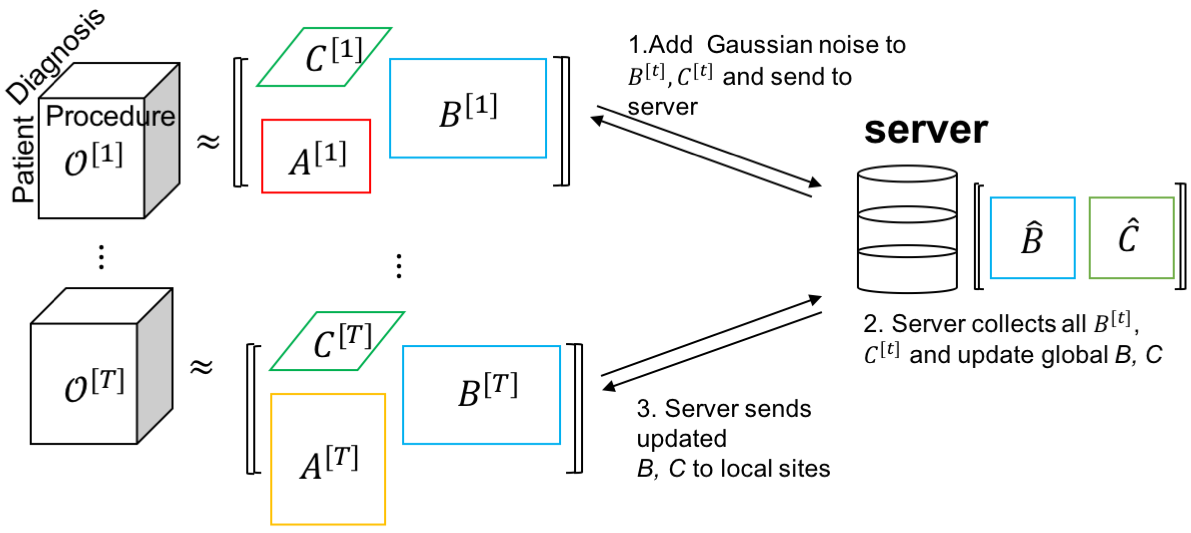}
\caption{Algorithm Overview}
\label{fig:algorithm}
\end{figure}

The patient data at each site is used to construct a local observed tensor,  $\mathcal{O}^{[t]}$.
For simplicity and illustration purposes, we discuss a three-mode tensor situation where the modes are patients, procedures, and diagnoses but \methodName~generalizes to $N$ modes.
The $T$ sites jointly decompose their local tensor into three factor matrices: a patient factor matrix $\textbf{A}^{[t]}$ and two feature factor matrices $\textbf{B}^{[t]}$ and $\textbf{C}^{[t]}$.
We assume that the factor matrices on the non-patient modes (i.e., $\textbf{B}^{[t]}, \textbf{C}^{[t]}$) are the same across the $T$ sites, thus sharing the same computational phenotypes.
To achieve consensus of the shared factor matrices, the non-patient feature factor matrices are shared in a privacy-preserving manner with the central server by adding Gaussian noise to each uploaded factor matrix.

Although the collaborative tensor problem for computational phenotyping has been previously discussed \cite{kim2017federated}, \methodName~provides three important contributions:

\textbf{(1) Efficiency}: We adopt a communication-efficient stochastic gradient descent (SGD) algorithm for collaborative learning which allows each site to transmit less information to the centralized server while still achieving an accurate decomposition. 

\textbf{(2) Heterogeneity}: A traditional global consensus model requires learning the same shared model from multiple sources. However, different data sources may have distinct patterns and properties (e.g., disease prevalence may differ between Georgia and Texas). We propose using the $l_{2,1}$-norm to achieve global consensus among the sites 
while capturing site-specific factors.

\textbf{(3) Differential Privacy Guarantees}: We preserve the privacy of intermediary results by adding Gaussian noise to each non-patient factor matrix prior to sharing with the parameter server. This masks any particular entry in the factor matrices and prevents inadvertent privacy leakage.
A rigorous privacy analysis based on zCDP is performed to ensure strong privacy protection for the patients.

\subsection{Formulation}

Under a single (centralized) model, CP decomposition of the observed tensor $\mathcal{O}$ results in a factorized tensor $\mathcal{X}$ that contains the $R$ most prevalent computational phenotypes.
We represent the centralized tensor as $T$ separate horizontal partitions, $\mathcal{O}^{[1]}, \cdots, \mathcal{O}^{[T]}$.
Thus, the global function can be expressed as the sum of $T$ separable functions with respect to each local factorized tensor $\mathcal{X}^{[t]}$  \cite{kim2017federated}:
\begin{equation}
\min\limits_{\mathcal{X}} \mathcal{L}= \frac{1}{2} || \mathcal{O} - \mathcal{X}||_F^2 = \sum_{t=1}^{T}{\frac{1}{2}} {\left\|{\mathcal{O}}^{[t]} - {\mathcal{X}}^{[t]}\right\|}_F^2. 
\label{eq:opti_deco}
\end{equation}
Since the goal is to uncover computational phenotypes that are shared across all sites, we restrict the sites to factorize the observed local tensors $\mathcal{O}^{[t]}$ such that the non-patient factor matrices are the same. 
Therefore, the global optimization problem is formulated as:
\begin{align*}
\min ~&\sum_{t=1}^{T}{\frac{1}{2}} {\left\|{\mathcal{O}}^{[t]} - {[\![\textbf{A}^{[t]},\textbf{B}^{[t]},\textbf{C}^{[t]}]\!]}\right\|}_F^2 \\
\text{s.t.}~& \textbf{B}^{[1]} = \textbf{B}^{[2]}=\cdots =\textbf{B}^{[T]} \\
& \textbf{C}^{[1]} = \textbf{C}^{[2]}=\cdots =\textbf{C}^{[T]}.
\end{align*}

This can be reformulated as a global consensus optimization, which decomposes the original problem into $T$ local subproblems by introducing two auxiliary variables, $\hat{\mathbf{B}}, \hat{\mathbf{C}}$, to represent the global factor matrices.
A quadratic penalty is placed between the local and global factor matrices to achieve global consensus among the $T$ different sites.
Thus, the local optimization problem at site $t$ is:
\begin{equation}
\begin{aligned}
\min~& \frac{1}{2} {\left\|{\mathcal{O}}^{[t]} - {[\![\textbf{A}^{[t]},\textbf{B}^{[t]},\textbf{C}^{[t]}]\!]}\right\|}_F^2  \\
& + \frac{\gamma}{2}{\left\| \textbf{B}^{[t]} - \hat{\textbf{B}}\right\|}_F^2 + \frac{\gamma}{2}{\left\| \textbf{C}^{[t]} - \hat{\textbf{C}}\right\|}_F^2.
\end{aligned}
\label{eq:opti}
\end{equation}

\subsection{Heterogeneous Patient Populations}

The global consensus model assumes that the patient populations are the same across different sites.
However, this may be too restrictive as some locations can have distinctive patterns.
For example, patients from the cardiac coronary unit may have unique characteristics that are different from the surgical care unit.
\methodName~utilizes the $l_{2,1}$-norm regularization, to allow flexibility for each site to ``turn off" one or more computational phenotypes.
For an arbitrary matrix $\textbf{W} \in \mathbb{R}^{m\times n}$, its $l_{2,1}$-norm is defined as:
\begin{equation}
{\lVert \textbf{W} \rVert}_{2,1}=\sum \limits_{i=1}^{m}  \sqrt{\sum_{j=1}^{n} \textbf{W}_{ij}^2} \label{eq:l21}.
\end{equation}
From the definition, we can see that the $l_{2,1}$-norm controls the row sparsity of matrix $\textbf{W}$.
As a result, the $l_{2,1}$-norm is commonly used in multi-task feature learning to perform feature selection as it can induce structural sparsity  \cite{guo2013probabilistic, liu2009multi, yang2011l2, nie2010efficient}.

\methodName~adopts a multi-task perspective, where each local decomposition is viewed as a separate task.
Under this approach, each site is not required to be characterized by all $R$ computational phenotypes.
To achieve this, we introduce the $l_{2,1}$-norm on the transpose of the patient factor matrices, $\textbf{A}^{[t]}$, to induce sparsity on the columns.
The idea is that if a specific phenotype is barely present in any of the patients (2-norm of the column is close to 0), the regularization will encourage all the column entries to be 0.
This can be used to capture the heterogeneity in the patient populations without violating the global consensus assumption.
Thus the \methodName~optimization problem is: 
\begin{equation}
\begin{aligned}
\min~& \sum_{t=1}^T ( \frac{1}{2} {\left\|{\mathcal{O}}^{[t]} - {[\![\textbf{A}^{[t]},\textbf{B}^{[t]},\textbf{C}^{[t]}]\!]}\right\|}_F^2  + \frac{\gamma}{2}{\left\| \textbf{B}^{[t]} - \hat{\textbf{B}}\right\|}_F^2 \\
&  + \frac{\gamma}{2}{\left\| \textbf{C}^{[t]} - \hat{\textbf{C}}\right\|}_F^2 + \mu{\left\|(\textbf{A}^{[t]})^\top\right\|_{2,1}} ).
\end{aligned}
\label{eq:norm}
\end{equation}
The quadratic penalty, $\gamma$, provides an elastic force to achieve global consensus between the local factor matrices and the global factor matrices whereas the $l_{2,1}$-norm penalty, $\mu$, encourages sites to share similar sparsity patterns.

\section{\methodName~Optimization}

\methodName~adopts the Elastics Averaging SGD (EASGD) \cite{zhang2015deep} approach to solve the optimization problem \eqref{eq:norm}.
EASGD is a communication-efficient algorithm for collaborative learning and has been shown to be more stable than the Alternating Direction Method of Multipliers (ADMM) with regard to parameter selection.
Moreover, SGD-based approaches scale well to sparse tensors, as the computation is bounded by the number of non-zeros. 

Using the EASGD approach, the global consensus optimization problem is solved alternatively between the local sites and the central server.
Each site performs multiple rounds of local tensor decomposition and updates their local factor matrices.
The site then only shares the most updated non-patient mode matrices with output perturbation to prevent revealing of sensitive information.
The patient factor matrix is never shared with the central server to avoid direct leakage of patient membership information.
The server then aggregates the updated local factor matrices to update the global factor matrices and 
sends the new global factor matrices back to each site.
This process is iteratively repeated until there are no changes in the local factor matrices.
The entire \methodName~decomposition process is summarized in Algorithm \ref{algorithm}.

\begin{algorithm}[ht]
\SetAlgoLined
\caption{\methodName}
\label{algorithm}
\KwIn{$\mathcal{O}$, $\tau$ $\eta$, $\gamma$, $\mu$, $\sigma$, $\rho$.}
Randomly initialize the global feature factor matrices $\textbf{B}$, $\textbf{C}$ and local feature  factor matrices $\textbf{B}^{[t]}$, $\textbf{C}^{[t]}$.\\
\While{$\textbf{B}^{[t]}$, $\textbf{C}^{[t]}$ not converge} {%
    \If{$\textbf{Hospital}$} {%
        \For{k = $1, \cdots, \tau$} {
        Shuffle tensor elements; \\
        \For{observation i} {
        Update $\textbf{A}^{[t]}$ using \eqref{eq:ai}; \\
        Update $\textbf{B}^{[t]}$, $\textbf{C}^{[t]}$ using \eqref{eq:bj};\\
        }
        Proximal update for $\prescript{}{new}{\textbf{A}^{[t]}}$ using \eqref{eq:an2}; \\
        }
        Calibrate $Gaussian$ noise matrix $\mathcal{M}_B^{[t]}$ and $\mathcal{M}_C^{[t]}$ as $\mathcal{N}$ $(0,  {\Delta_{2}^{2} /(2 \rho)})$ for each factor matrix; \\
        Update factor matrices $\prescript{}{priv}{\textbf{B}^{[t]}}$ and $\prescript{}{priv}{\textbf{C}^{[t]}}$ using \eqref{eq:priv};\\
        
        Send $\prescript{}{priv}{\textbf{B}^{[t]}}$, $\prescript{}{priv}{\textbf{C}^{[t]}}$ to Server.
            
        }
        \If{$\textbf{Server}$}{%
            Collect $\prescript{}{priv}{\textbf{B}^{[t]}}$, $\prescript{}{priv}{\textbf{C}^{[t]}}$ from each hospital;\\
            Update $\widehat{\textbf{B}}$, $\widehat{\textbf{C}}$ using \eqref{eq: global};\\
            Send $\widehat{\textbf{B}}$, $\widehat{\textbf{C}}$ back to hospitals.
        } 
}
\end{algorithm}

\subsection{Local Factors Update}

Each site updates the local factors by solving the following subproblem:
\begin{equation}
\begin{aligned}
\min~&  \frac{1}{2} {\left\|{\mathcal{O}}^{[t]} - {[\![\textbf{A}^{[t]},\textbf{B}^{[t]},\textbf{C}^{[t]}]\!]}\right\|}_F^2  + \frac{\gamma}{2}{\left\| \textbf{B}^{[t]} - \hat{\textbf{B}}\right\|}_F^2 \\
&  + \frac{\gamma}{2}{\left\| \textbf{C}^{[t]} - \hat{\textbf{C}}\right\|}_F^2 + \mu{\left\|(\textbf{A}^{[t]})^\top\right\|_{2,1}}.
\end{aligned}
\label{eq:local-subproblem}
\end{equation}
EASGD helps reduce the communication cost by allowing sites to perform multiple iterations (each iteration is one pass of the local data) before sending the updated factor matrices.
We further extend the local optimization updates using permutation-based SGD (P-SGD), a practical form of SGD \cite{wu2017bolt}.
In P-SGD, instead of randomly sampling one instance from the tensor at a time, the non-zero elements are first shuffled within the tensor.
The algorithm then cycles through these elements to update the latent factors.
At each local site, the shuffling and cycling process is repeated $\tau$ times, hereby referred to as a $\tau$-pass P-SGD.
There are two benefits of adopting the P-SGD approach: 1) the resulting algorithm is more computationally effective as it eliminates some of the randomness of the basic SGD algorithm.
2) it provides a mechanism to properly estimate the total privacy budget (see Section \ref{sec:priv}).

\subsubsection{Patient Factor Matrix}
\label{sec:pfm}

For site $t$, the patient factor matrix $\textbf{A}^{[t]}$ is updated by minimizing the objective function using the local factorized tensor, $\mathcal{X}^{[t]}$ and the $l_{2,1}$-norm:
\begin{equation}
\min\limits_{\textbf{A}^{[t]}} \underbrace{{\frac{1}{2}}{\left\| \mathcal{O}^{[t]}-[\![\textbf{A}^{[t]}, \textbf{B}^{[t]}, \textbf{C}^{[t]}]\!] \right\|}_F^2}_{\mathcal{F}} + \underbrace{\mu{\left\|(\textbf{A}^{[t]})^\top\right\|_{2,1}}}_{\mathcal{H}} \label{eq:patnorm}.
\end{equation}
While the $l_{2,1}$-norm is desirable from a modeling perspective, it also results in a non-differentiable optimization problem.
The local optimization problem \eqref{eq:patnorm} can be seen as a combination of a differentiable function $\mathcal{F}$ and a non-differentiable function $\mathcal{H}$.
Thus, we propose using the proximal gradient descent method to solve local optimization problem for the patient mode.
Proximal gradient method can be applied in our case since the gradient of the differentiable function $\mathcal{F}$ is $Lipschitz$ continuous with a $Lipschitz$ constant $L$ (see Appendix for details).

Using the proximal gradient method, the factor matrix $\textbf{A}^{[t]}$ is iteratively updated via the proximal operator:
\begin{equation}
\prescript{}{new}{\textbf{A}^{[t]}}=\textbf{prox}_{\eta \mathcal{H}}\left(\textbf{A}^{[t]}-\eta \nabla\mathcal{F}(\textbf{A}^{[t]})\right),
\end{equation}
where $\eta>0$ is the step size at each local iteration.
The proximal operator is computed by solving the following equation:
\begin{equation}
\textbf{prox}_{\eta \mathcal{H}}(\Theta)=\arg\min\limits_{\Theta}\left({\frac{1}{2\eta}}{\lVert \Theta-\hat{\Theta}\rVert}+\mathcal{H}(\Theta)\right),\label{eq:prox1}
\end{equation}
where $\hat{\Theta}=\textbf{A}^{[t]}-\eta \nabla\mathcal{F}(\textbf{A}^{[t]})$ is the updated matrix.
It has been shown that if $\nabla\mathcal{F}$ is $Lipschitz$ continuous with constant $L$, the proximal gradient descent method will converge for step size $\eta<2/L$ \cite{combettes2011proximal}.  
For the $l_{2,1}$-norm, the closed form solution can be computed using the soft-thresholding operator:
\begin{equation}
\textbf{prox}_{\eta \mathcal{H}}(\widehat{{\Theta}})=\widehat{{\Theta}_{r:}}{\left(1-{\frac{\mu}{\lVert \widehat{{\Theta}_{r:}} \rVert}_2}\right)}_+,\label{eq:prox}
\end{equation}
where $r\in (0, R]$ and $r$ represents the $r$-th column of the factor matrix $\widehat{\Theta}$, and $(z)_{+}$ denotes the maximum of 0 and $z$.
Thus, if the norm of the $r$-th column of the patient matrix is small, the proximal operator will ``turn off" that column.

The gradient of the smooth part can be derived with respect to each row in the patient mode factor matrix, $\textbf{A}^{[t]}$.
The update rule for each row is:
\begin{equation}
\textbf{a}_{i:}^{[t]}\gets \textbf{a}_{i:}^{[t]}-\eta\left[\left(\textbf{a}_{i:}^{[t]}{(\textbf{b}_{j:}^{[t]}*\textbf{c}_{k:}^{[t]})}^\top -\mathcal{O}_{ijk}^{[t]}\right)\left(\textbf{b}_{j:}^{[t]}*\textbf{c}_{k:}^{[t]}\right)\right] \label{eq:ai}
\end{equation}

After one pass through all entries in a local tensor to update the patient factor matrix, the second step is to use proximal operator \eqref{eq:prox} to update the patient factor matrix $\textbf{A}^{[t]}$:
\begin{equation}
\begin{aligned}
&\prescript{}{new}{\textbf{A}^{[t]}}=\textbf{prox}_{\eta \mathcal{H}}(\textbf{A}^{[t]}).\\
\end{aligned}
\label{eq:an2}
\end{equation}

\subsubsection{Feature Factor Matrices}

The local feature factor matrices, $\textbf{B}^{[t]}$ and $\textbf{C}^{[t]}$, are updated based on the following objective functions:
\begin{equation}
\begin{aligned}
&\min \limits_{\textbf{B}^{[t]}} f_b =  {\frac{1}{2}}{\left\| \mathcal{O}^{[t]}-[\![\textbf{A}^{[t]}, \textbf{B}^{[t]}, \textbf{C}^{[t]}]\!]\right\|}_F^2+{\frac{\gamma}{2}}{\left\| \textbf{B}^{[t]} - \hat{\textbf{B}}\right\|}_F^2,\\
&\min \limits_{\textbf{C}^{[t]}} f_c =  {\frac{1}{2}}{\left\| \mathcal{O}^{[t]}-[\![\textbf{A}^{[t]}, \textbf{B}^{[t]}, \textbf{C}^{[t]}]\!]\right\|}_F^2+{\frac{\gamma}{2}}{\left\| \textbf{C}^{[t]} - \hat{\textbf{C}}\right\|}_F^2.
\end{aligned}
\label{eq:obj_feature}
\end{equation}
The partial derivatives of $f_b, f_c$ with respect to $\textbf{b}_{j:}^{[t]}$ and $\textbf{c}_{k:}^{[t]}$, the $j$-th and $k$-th row of the $\textbf{B}^{[t]}$ and $\textbf{C}^{[t]}$ factor matrices, respectively, are computed.

\begin{equation}
\begin{aligned}
&{\frac{\partial f_b}{\partial \textbf{b}_{j:}^{[t]}}}=
\left[\left(\textbf{a}_{i:}^{[t]}{(\textbf{b}_{j:}^{[t]}*\textbf{c}_{k:}^{[t]})}^\top -\mathcal{O}_{ijk}^{[t]}\right)\left(\textbf{a}_{i:}^{[t]}*\textbf{c}_{k:}^{[t]}\right)\right]\\
&{\frac{\partial f_c}{\partial \textbf{c}_{k:}^{[t]}}}=
\left[\left(\textbf{a}_{i:}^{[t]}{(\textbf{b}_{j:}^{[t]}*\textbf{c}_{k:}^{[t]})}^\top -\mathcal{O}_{ijk}^{[t]}\right)\left(\textbf{a}_{i:}^{[t]}*\textbf{b}_{j:}^{[t]}\right)\right].
\end{aligned}
\label{eq:feature-update}
\end{equation}
$\textbf{B}^{[t]}$ and $\textbf{C}^{[t]}$ are then updated row by row by adding up the partial derivative of the quadratic penalty term and the partial derivative with respect to $\textbf{b}_{j:}^{[t]}$ and $\textbf{c}_{k:}^{[t]}$ shown in \eqref{eq:feature-update}.

\begin{equation}
\begin{aligned}
&\textbf{b}_{j:}^{[t]}\gets \textbf{b}_{j:}^{[t]}-\eta\left[{\frac{\partial f_n}{\partial \textbf{b}_{j:}^{[t]}}} + \gamma \left( \textbf{b}_{j:}^{[t]} - \widehat{\textbf{b}_{j:}}\right)\right];\\
&\textbf{c}_{k:}^{[t]}\gets \textbf{c}_{k:}^{[t]}-\eta\left[{\frac{\partial f_n}{\partial \textbf{c}_{k:}^{[t]}}} +\gamma \left( \textbf{c}_{k:}^{[t]} - \widehat{\textbf{c}_{k:}}\right)\right].
\end{aligned}
\label{eq:bj}
\end{equation}

Each site simultaneously does several rounds ($\tau$) of the local factor updates.
After $\tau$ rounds are completed, the feature factor matrices will be perturbed with $Gaussian$ noise and sent to central server.

\subsubsection{Privacy-Preserving Output Perturbation}
Although the feature factor matrices do not directly contain patient information, it may inadvertently violate patient privacy (e.g., a rare disease that is only present in a small number of patients). 
To protect the patient information from being speculated by semi-honest server, we perturb the feature mode factor matrices using the Gaussian mechanism, a common building block to perturb the output and achieve rigorous differential privacy guarantee.

The Gaussian mechanism adds zero-mean Gaussian noise with standard deviation $\sigma={\Delta_{2}^{2} /(2 \rho)}$ to each element of the output \cite{bun2016concentrated}.
Thus, the noise matrix $\mathcal{M}$ can be calibrated for each factor matrices $\textbf{B}^{[t]}$ and $\textbf{C}^{[t]}$ based on their $L_2$-sensitivity to construct privacy-preserving feature factor matrices:

\begin{equation}
\begin{aligned}
&\prescript{}{priv}{\textbf{B}^{[t]}} \gets{\textbf{B}^{[t]}}+\mathcal{M}_B^{[t]},\\
&\prescript{}{priv}{\textbf{C}^{[t]}} \gets{\textbf{C}^{[t]}}+\mathcal{M}_C^{[t]},\\
\end{aligned}
\label{eq:priv}
\end{equation}
As a result, each factor matrix that is shared with the central server satisfies $\rho$-zCDP by Proposition 2.7.
A detailed privacy analysis for the overall privacy guarentee is provided in the next subsection.

\subsection{Privacy Analysis}
\label{sec:priv}
In this section we analyze the overall privacy guarantee of Algorithm \ref{algorithm}. 
The analysis is based on the following knowledge of the optimization problem: 1) each local site performs a $\tau$-pass P-SGD update per epoch; 2) for the local objective function $f$ in \eqref{eq:obj_feature}, when fixing two of the factor matrices, the objective function becomes a convex optimization problem for the other factor matrix.

\subsubsection{$L_2$-sensitivity}
The objective function \eqref{eq:obj_feature} satisfies $L-Lipschitz$, with $Lipschitz$ constant $L$ the tight upper bound of the gradient.
For a $\tau$-pass P-SGD, having constant learning rate $\eta=\eta_{k}\leq{\frac{2}{\beta}}$ ($k=1,...,\tau$, $\beta$ is the $Lipschitz$ constant of the gradient of \eqref{eq:obj_feature} regarding $\textbf{B}^{[t]}$ or $\textbf{C}^{[t]}$, see Appendix for $\beta$ calculation), the $L_2$-sensitivity of this optimization problem in \eqref{eq:obj_feature} is calculated as $\Delta_2(f)=2\tau L\eta$ \cite{wu2017bolt}.

\subsubsection{Overall Privacy Guarantee}

The overall privacy guarantee of Algorithm 1 is analyzed under the zCDP definition which provides tighter privacy bound than strong composition theorem \cite{dwork2010boosting} for multiple folds Gaussian mechanism \cite{ bun2016concentrated, yu2019differentially}. The total $\rho$-zCDP will be transferred to $(\epsilon, \delta)$-DP in the end using Proposition 2.8.

\begin{theorem} \label{theorem:overall privacy guarantee}
Algorithm 1 is $(\epsilon, \delta)$-differentially private if we choose 
% $$
% \sigma=\frac{2 \sqrt{E \log (1 / \delta)}}{\epsilon}
% $$
the input privacy budget for each factor matrix per epoch as
$$
\rho=\frac{\epsilon^{2}}{8 E \log (1 / \delta)}
$$
where $E$ is the number of epochs when the algorithm is converged.
\end{theorem}

\begin{proof}
Let the \text{"base"} zCDP parameter be $\rho_{b}$, $\textbf{B}^{[t]}$ and $\textbf{C}^{[t]}$ together cost $2E\rho_{b}$ after $E$ epochs by Proposition 2.9. All $T$ user nodes cost $\frac{1}{T} \sum_{t=1}^{T} 2 E \rho_{b}=2 E \rho_{b}$ by the \textit{parallel composition theorem} in Proposition 2.10. By the connection of zCDP and $(\epsilon, \delta)$-DP in Proposition 2.8, we get $\rho_{b}=\frac{\epsilon^{2}}{8 E \log (1 / \delta)}$, which concludes our proof.
\end{proof}

\subsection{Global Variables Update}
The server receives $T$ local feature matrix updates, and then updates the global feature matrices according to the same objective function in \eqref{eq:opti}.
The gradient for the global feature matrices 
$\widehat{\textbf{B}}$ and $\widehat{\textbf{C}}$ are: 

\begin{equation}
\begin{aligned}
&\widehat{\textbf{B}}\gets \widehat{\textbf{B}}+\eta \sum\limits_{t=1}^{T}\gamma \left(\prescript{}{priv}{\textbf{B}^{[t]}}-\widehat{\textbf{B}}\right) \\
&\widehat{\textbf{C}}\gets \widehat{\textbf{C}}+\eta \sum\limits_{t=1}^{T}\gamma \left(\prescript{}{priv}{\textbf{C}^{[t]}}-\widehat{\textbf{C}}\right).
\end{aligned}
\label{eq: global}
\end{equation}
The update makes the global phenotypes similar to the local phenotypes at the $T$ local sites.
The server then sends the global information, $\widehat{\textbf{B}}, \widehat{\textbf{C}}$
to each site for the next epoch.

\section{Experimental Evaluation}

We evaluate \methodName~on three aspects: 1) efficiency based on accuracy and communication cost; 2) utility of the phenotype discovery; and 3) impact of privacy.
The evaluation is performed on both real-world datasets and synthetic datasets.

\subsection{Dataset}
We evaluated \methodName~on one synthetic dataset and two real-world datasets, MIMIC-III  \cite{johnson2016mimic} and the CMS DE-SynPUF dataset.
Each of the dataset has different sizes, sparsity (i.e., \% of non-zero elements), and skewness in distribution (i.e., some sites have more patients).

\noindent \textbf{MIMIC-III.}
This is a publicly-available intensive care unit (ICU) dataset collected from 2001 to 2012. We construct 6 local tensors with different sizes representing patients from different ICUs. 
Each tensor element represents the number of co-occurrence of diagnoses and procedures from the same patient within a 30-day time window.
For better interpretability, we adopt the rule in \cite{kim2017discriminative} and select 202 procedures ICD-9 codes and 316 diagnoses codes that have the highest frequency.
The resulting tensor is $40,662$ patients $\times$ 202 procedures $\times$ 316 diagnoses with a non-zero ratio of $4.0382\times 10^{-6}$.

\noindent \textbf{CMS.} This is a publicly-available Data Entrepreneurs' Synthetic Public Use File (DE-SynPUF) from 2008 to 2010. We randomly choose 5 samples out of the 20 samples of the outpatient data to construct 5 local tensors with patients, procedures and diagnoses. Different from MIMIC-III, we make each local tensor the same size. There are 82,307 patients with 2,532 procedures and 10,983 diagnoses within a 30-day time window. We apply the same rule in selecting ICD-9 codes. By concatenating the 5 local tensors, we obtain a big tensor with $3.1678\times 10^{-7}$ non-zero ratio.

\noindent \textbf{Synthetic Dataset.} We also construct tensors from synthetic data. In order to test different dimensions and sparsities, we construct a tensor of size $5000\times 300\times 800$ with a sparsity rate of $10^{-5}$ and then horizontally partition it into 5 equal parts.

\subsection{Baselines}

We compare our \methodName~framework with two centralized baseline methods and an existing state-of-the-art federated tensor factorization method as described below.   

\noindent \textbf{CP-ALS}: A widely used, centralized model that solves tensor decomposition using an alternating least squares approach. Data from multiple sources are combined to construct the global tensor.

\noindent \textbf{SGD}: A  centralized method that solves the tensor decomposition use the stochastic gradient descent-based approach.
This is equivalent to \methodName~with a single site and no regularization ($T=1,\gamma=0,\mu=0$).
We consider this a counterpart to the CP-ALS method.

\noindent \textbf{TRIP} \cite{kim2017federated}:
A federated tensor factorization framework that enforces a shared global model and does not offer any differential privacy guarantee. TRIP utilizes the consensus ADMM approach to decompose the problem into local subproblems. 

\begin{figure*}[htbp]
\setlength{\abovecaptionskip}{0cm}
\setlength{\belowcaptionskip}{0cm}
\centering
\subfigure[MIMIC-III: ICU]{
\includegraphics[width=2.24in,trim={0.5mm 0.5mm 0.5mm 0.5mm},clip]{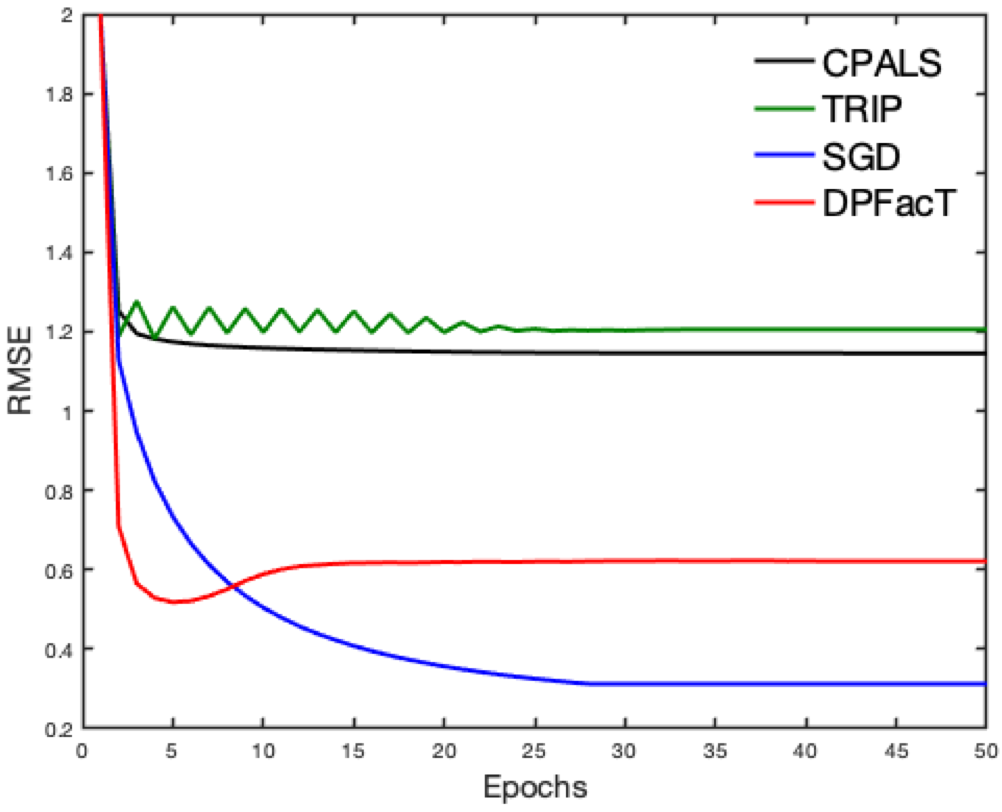}
}
\subfigure[CMS]{
\includegraphics[width=2.256in,trim={0.8mm 0.8mm 0.5mm 0.5mm},clip]{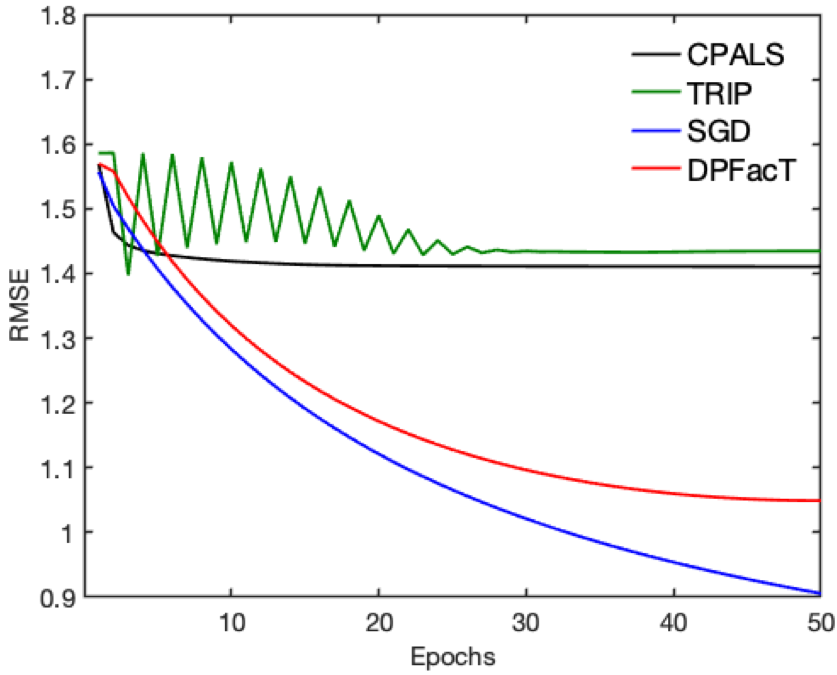}
}
\subfigure[Synthetic]{
\includegraphics[width=2.257in, trim={0.5mm 0 0 0},clip]{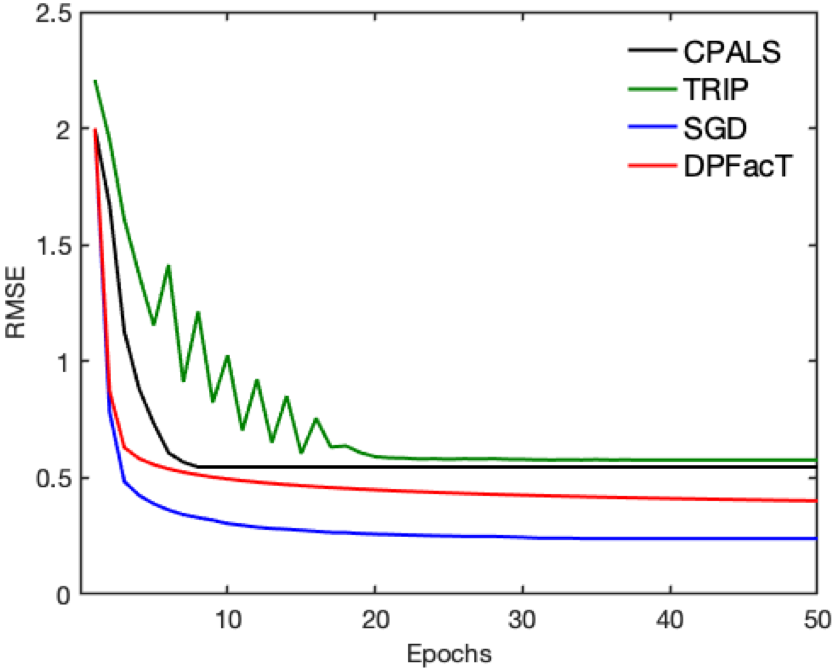}
}
\caption{Average RMSE on (a) MIMIC-III, (b) CMS, (c) Synthetic datasets using 5 random initializations.}
\label{fig:rmse}
\end{figure*}

\subsection{Implementation Details}
\methodName~is implemented in MatlabR2018b with the Tensor Toolbox Version 2.6 \cite{TTB_Software} for tensor computing and the Parallel Computing Toolbox of Matlab. The experiments were conducted on m5.4xlarge instances of AWS EC2 with 8 workers. For prediction task, we build the logistic regression model with Scikit-learn library of Python 2.7. For reproducibility purpose, we made our code publicly available\footnote{https://github.com/jma78/DPFact.}.

\subsection{Parameter Configuration}
Hyper-parameter settings include quadratic penalty parameter $\gamma$, $l_{2,1}$ regularization term $\mu$, learning rate $\eta$, and the input per-epoch, per-factor matrix privacy budget $\rho$. The rank $R$ is set to 50 to allow some site-specific phenotypes to be captured.

\subsubsection{Quadratic penalty parameter $\gamma$} 
The quadratic penalty term can be viewed as an elastic force between the local factor matrices and the global factor matrices. 
Smaller $\gamma$ allows more exploration of the local factors 
but will result in slower convergence. To balance the trade-off between convergence and stability, we choose $\gamma=5$ after grid search through $\gamma=\{2, 5, 8, 10\}$.

\subsubsection{$l_{2,1}$-regularization term $\mu$} 
We evaluate the performance of \methodName~with different $\mu$ for different ICU types as they differ in the $Lipschitz$ constants. Smaller $\mu$ has minimal effect on the column sparsity, as there are no columns that are set to 0, while higher $\mu$ will "turn off" a large portion of the factors and prevent \methodName~from generating useful phenotypes.
Based on figure \ref{fig:icu_norms}, we choose $\mu = \{1, 1.8, 3.2, 1.8, 1.5, 0.6\}$ for TSICU, SICU, MICU, CSRU, CCU, NICU respectively for MIMIC-III to maintain noticeable differences in the column magnitude and the flexibility to have at least one unshared column (see Appendix for details).
Similarly, we choose $\mu=2$ equally for each site for CMS and $\mu=0.5$ equally for each site for the synthetic dataset.

\subsubsection{Learning rate $\eta$} 
The learning rate $\eta$ must be the same for local sites and the parameter server. The optimal $\eta$ was found after grid searching in the range $[10^{-5}$, $10^{-1}]$. We choose $10^{-2}, 10^{-3}$, and $10^{-2}$ for MIMIC-III, CMS, and synthetic data respectively.

\subsubsection{Privacy budget $\rho$}  
We choose the per-epoch privacy budget under the zCDP definition for each factor matrix as $\rho=10^{-3}$ for MIMIC-III, CMS, and synthetic dataset. By Theorem 4.1, the total privacy guarantee is $(1.2, 10^{-4})$, $(1.9, 10^{-4})$, and $(1.7, 10^{-4})$ under the $(\epsilon, \delta)$-DP definition for MIMIC-III, CMS, and synthetic dataset respectively when DPFact converges (we choose $\delta$ to be $10^{-4}$).

\subsubsection{Number of sites $T$}
To gain more knowledge on how communication cost would be reduced regarding the number of sites, we evaluate the communication cost when the number of sites ($T$) are increased.
To simulate a larger number of sites, we randomly partition the global observed tensor into 1, 5, and 10 sites for the three datasets. 
Table \ref{tab:commu2} shows that the communication cost of \methodName~scales proportionally with the number of sites.

\begin{table}
  \setlength{\abovecaptionskip}{0cm}
  \setlength{\belowcaptionskip}{0cm}
  \centering
  \begin{tabular}{l r r r}
  \toprule
    \textbf{\# of Sites} & \textbf{MIMIC-III} & \textbf{CMS} & \textbf{Synthetic}\\
    \midrule
    1 & 18.73 & 22.89 & 1.55\\
    5 & 93.62 & 114.42 & 7.75\\
    10 & 189.83  & 228.83 & 15.50\\
    \bottomrule
\end{tabular}
\caption{Communication cost of \methodName~for different number of sites (Seconds)}
\label{tab:commu2}
\end{table}

\subsection{Efficiency}
\subsubsection{Accuracy}
Accuracy is evaluated using the root mean square error (RMSE) between the global observed tensor and a horizontal concatenation of each factorized local tensor.
Figure \ref{fig:rmse} illustrates the RMSE as a function of the number of epochs. 
We observe that \methodName~converges to a smaller RMSE than CP-ALS and TRIP.
SGD achieves the lowest RMSE as \methodName~ suffers some utility loss by sharing differentially private intermediary results.

\subsubsection{Communication Cost}
The communication cost is measured based on the total number of communicated bytes divided by the data transfer rate (assumed as 15 MB/second).
As CP-ALS and SGD are both centralized models, only TRIP and \methodName~are compared.

Table \ref{table:commu1} summarizes the communication cost on all the datasets.
\methodName~reduces the cost by 46.6\%, 37.7\%, and 20.7\% on MIMIC-III, CMS, and synthetic data,  respectively.
This is achieved by allowing more local exploration at each site (multiple passes of the data) and transmitting fewer auxiliary variables.
Moreover, the reduced communication cost does not result in higher RMSE (see Figure \ref{fig:rmse}).

\begin{table}
  \setlength{\abovecaptionskip}{0cm}
  \setlength{\belowcaptionskip}{0cm}
  \centering
  \begin{tabular}{ l r r r}
  \toprule
    \textbf{Algorithm} & \textbf{MIMIC-III} & \textbf{CMS} & \textbf{Synthetic}\\
    \midrule
    TRIP & 175.26 & 183.72 & 9.77\\
    \methodName & 93.62 & 114.42 & 7.75\\
    \bottomrule
\end{tabular}
\caption{Communication Cost of \methodName~and TRIP (Seconds)}
\label{table:commu1}
\end{table}

\subsection{Utility}
The utility of \methodName~is measured by the predictive power of the discovered phenotypes.
A logistic regression model is fit using the patients' membership values (i.e., $\textbf{A}^{[t]}_{i:}$, $\widehat{\textbf{A}_{i:}}$ of size $1\times R$) as features to predict in-hospital mortality.
We use a 60-40 train-test split and evaluated the model using area under the receiver operating characteristic curve (AUC).

\subsubsection{Global Patterns}
Table \ref{table:auc} shows the AUC for \methodName, CP-ALS (centralized), and TRIP (distributed) as a function of the rank ($R$).
From the results, we observe that \methodName~outperforms both baseline methods for achieving the highest AUC.
This suggests that \methodName~captures similar global phenotypes as the other two methods.
We note that \methodName~has a slightly lower AUC than CP-ALS for a rank of 10, as the $l_{2,1}$-regularization effect is not prominent.

\begin{table}
  \setlength{\abovecaptionskip}{-0.1cm}
  \setlength{\belowcaptionskip}{0cm}
  \centering
  \begin{tabular}{ l c c c c c}
  \toprule
    \multirow{2}*{\textbf{Rank}} & \multirow{2}*{\textbf{CP-ALS}} & \multirow{2}*{\textbf{TRIP}} & \multicolumn{3}{c}{\textbf{\methodName}} \\
    \cmidrule(lr){4-6}
    & & & \textbf{\methodName} & \textbf{w/o $l_{2,1}$} & \textbf{w/o DP}\\
    \midrule
    10 & 0.7516 & 0.7130 & 0.7319 & 0.5189 & 0.7401\\
    20 & 0.7573 & 0.7596 & 0.7751 & 0.6886 & 0.7763\\
    30 & 0.7488 & \underline{0.7644} & 0.7679 & 0.6977 & 0.7705\\
    40 & 0.7603 & 0.7574 & 0.7737 & 0.7137 & 0.7756\\
    50 & 0.7643 & 0.7633 & \underline{0.7759} & 0.7212 & \underline{0.7790}\\
    60 & \underline{0.7648} & 0.7588 & 0.7758 & \underline{0.7312} & 0.7763\\
    \bottomrule
\end{tabular}
\caption{Predictive performance (AUC) comparison for (1) CP-ALS, (2) TRIP, (3) \methodName, (4) \methodName~without $l_{2,1}$-norm (w/o $l_{2,1}$), (5) non-private \methodName ~(w/o DP).}
\label{table:auc}
\end{table}

\subsubsection{Site-Specific Patterns}
Besides achieving the highest predictive performance, \methodName~also can be used to discover site-specific patterns.
As an example, we focus on the neonatal ICU (NICU) which has a drastically different population than the other 5 ICUs.
The ability to capture NICU-specific phenotypes can be seen in the AUC comparison with TRIP (Figure \ref{fig:nicu_auc}(a)).
\methodName~consistently achieves higher AUC for NICU patients.
The importance of the $l_{2,1}$-regularization term is also illustrated in Table \ref{table:auc}.
\methodName~with the $l_{2,1}$-regularization is more stable and achieves higher AUC compared without the regularization term ($\mu=0$).

Table \ref{table:nicu_stats} illustrates the top 5 phenotypes with respect to the magnitude of the logistic regression coefficient (mortality risk related to the phenotype) for NICU. The phenotypes are named according to the non-zero procedures and diagnoses.
A high $\lambda$ and prevalence means this phenotype is common.
From the results, we observe that heart disease, respiration failure, and pneumonia are more common but less associated with mortality risk (negative coefficient).
However, acute kidney injury (AKI) and anemia are less prevalent and highly associated with death.
In particular, AKI has the highest risk of in-hospital death, which is consistent with other reported results \cite{youssef2015incidence}. Table \ref{table:dp}(a) shows an NICU-specific phenotype, which differs slightly from the corresponding global phenotype showing in table \ref{table:dp}(b).

\begin{figure}[htbp]
\setlength{\abovecaptionskip}{0cm}
\setlength{\belowcaptionskip}{-0.25cm}
\centering
\subfigure[AUC comparison]{
\includegraphics[width=1.6in,trim={0.5mm 0.5mm 0.5mm 0.5mm},clip]{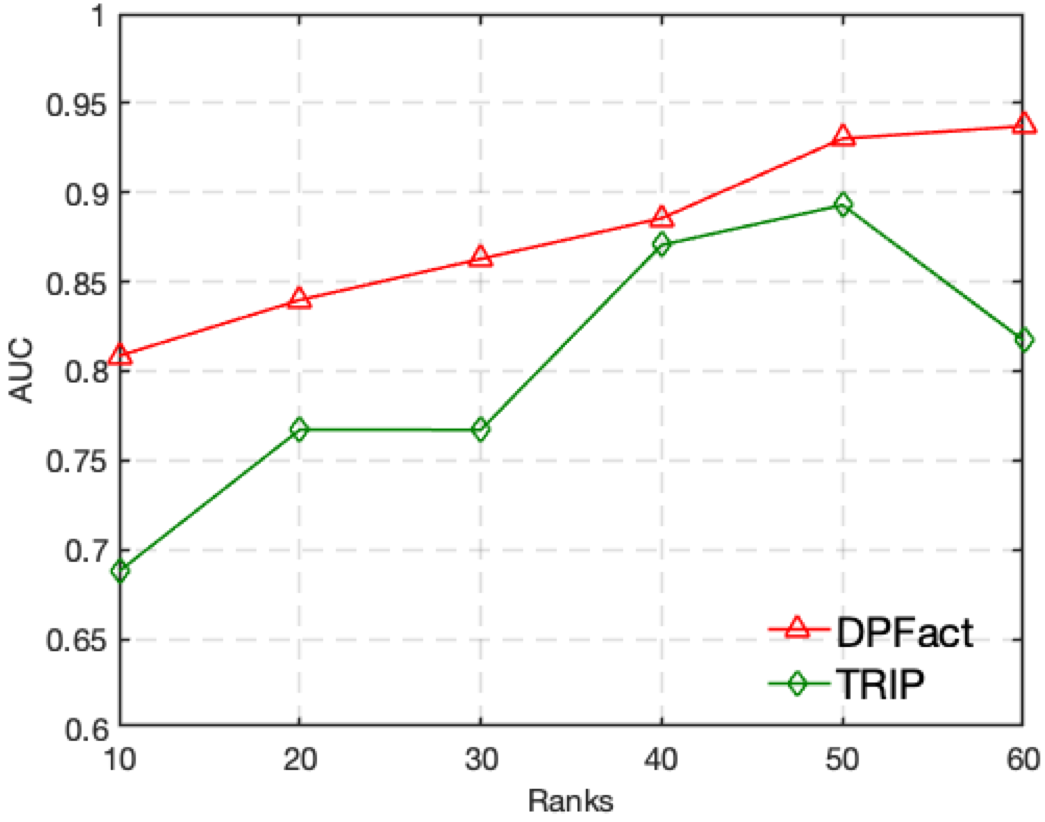}
}
\subfigure[Factor Match Score (FMS)]{
\includegraphics[width=1.55in,trim={1mm 2mm 1mm 1mm},clip]{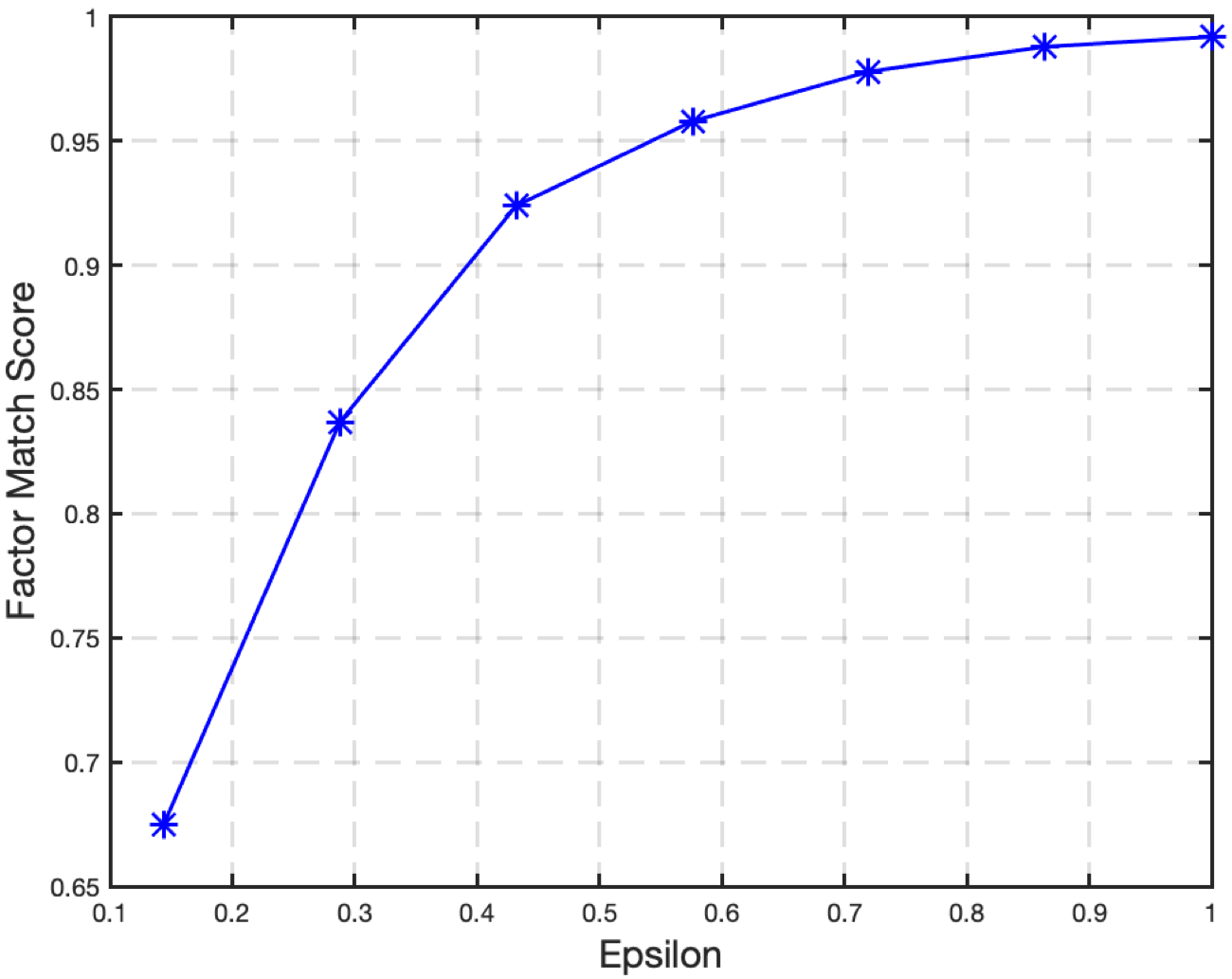}
}
\caption{(a) Predictive performance (AUC) comparison for NICU between (1) TRIP, (2) \methodName. (b) Factor Match Score (FMS) under different privacy budget $(\epsilon)$.}
\label{fig:nicu_auc}
\end{figure}

\begin{table}
  \setlength{\abovecaptionskip}{0cm}
  \setlength{\belowcaptionskip}{0cm}
  \centering
  \begin{tabular}{ p{3cm} r r r r}
  \toprule
    \textbf{Phenotypes} & \textbf{Coef} & \textbf{p-value} & \textbf{$\lambda$} & \textbf{Prevalence}\\
    \midrule
    25: Congenital heart defect & -2.1865 & 0.005 & 198 & 34.32\\
    29: Anemia & 3.5047 & <0.001 & 77 & 13.22\\
    30: Acute kidney injury & \textbf{5.8806} & <0.001 & 68 & 23.38\\
    34: Pneumonia & -5.1050 & <0.001 & 37 & 37.58\\
    35: Respiratory failure & -0.9141 & <0.001 & 85 & 24.40\\
    \bottomrule
\end{tabular}
\caption{Top 5 representative phenotypes from NICU based on the factor weights, $\lambda_r={\left\|{\textbf{A}_{:r}}\right\|}_{F}{\left\|{\textbf{B}_{:r}}\right\|}_{F}{\left\|{\textbf{C}_{:r}}\right\|}_{F}$. Prevalence is the proportion of patients who have non-zero membership to the phenotype.}
\label{table:nicu_stats}
\end{table}

\begin{table}
  \setlength{\abovecaptionskip}{0cm}
  \setlength{\belowcaptionskip}{0cm}
  \small
  \centering
  \renewcommand\arraystretch{0.3}
  \subtable[NICU-specific Phenotypes discovered by \methodName]{
    \begin{tabular}{|p{4cm}|p{4cm}|}
    \hline
    \textbf{Procedures} & \textbf{Diagnoses}\\
    \hline
    {\color{red} Cardiac catheterization} & {\color{blue} Ventricular fibrillation}\\
    {\color{red} Insertion of non-drug-eluting coronary artery stent(s)} & {\color{blue} Unspecified congenital anomaly of heart}\\
    {\color{red} Prophylactic administration of vaccine against other disease} & {\color{blue} Benign essential hypertension}\\
    \hline
   \end{tabular}
  }
  \renewcommand\arraystretch{0.3}
  \subtable[Globally shared phenotype discovered by \methodName]{
    \begin{tabular}{|p{4cm}|p{4cm}|}
    \hline
    \textbf{Procedures} & \textbf{Diagnoses}\\
    \hline
    {\color{red} Attachment of pedicle or flap graft} & {\color{blue} Rheumatic heart failure}\\
    {\color{red} Right heart cardiac catheterization} & {\color{blue} Ventricular fibrillation}\\
    {\color{red} Procedure on two vessels} & {\color{blue} Benign essential hypertension}\\
    {\color{red} Other endovascular procedures on other vessels} & {\color{blue} Paroxysmal ventricular tachycardia} {\color{blue} Nephritis and nephropathy}\\
    {\color{red} Insertion of non-drug-eluting coronary artery stent(s)} & \\
    \hline
   \end{tabular}
  }
  \renewcommand\arraystretch{0.3}
  \subtable[Globally shared phenotype discovered by non-private \methodName]{
   \begin{tabular}{|p{4cm}|p{4cm}|}
    \hline
    \textbf{Procedures} & \textbf{Diagnoses}\\
    \hline
    {\color{red} Right heart cardiac catheterization} & {\color{blue} Hypopotassemia} \\ 
    {\color{red} Attachment of pedicle or flap graft} & {\color{blue} Rheumatic heart failure}\\
    {\color{red} Excision or destruction of other lesion or tissue of heart, open approach} & {\color{blue} Benign essential hypertension \qquad } \qquad {\color{blue} Paroxysmal ventricular tachycardia} {\color{blue} Systolic heart failure}\\
    \hline
   \end{tabular}
  }
  \caption{Example of the representative phenotype. (a) NICU-specific phenotype of Congenital heart defect; (b) and (c) are the globally shared phenotype of Heart failure, showing the difference of \methodName~and non-private \methodName.}
  \label{table:dp}
\end{table}

\subsection{Privacy}

We investigated the impact of differential privacy by comparing \methodName~with its non-private version.
The main difference is that non-private \methodName~does not perturb the local feature factor matrices that are transferred to the server.
We use the factor match score (FMS) \cite{chi2012tensors} to compare the similarity between the phenotype discovered using \methodName~and non-private \methodName.
FMS defined as:
$$
score(\Bar{\mathcal{X}}) = {\frac{1}{R}}\sum\limits_r {\left( 1-{\frac{\left| \xi_r- \Bar{\xi_r} \right|}{max\left\{\xi_r, \Bar{\xi_r} \right\}}} \right)}\prod\limits_{\textbf{x}=\textbf{a,b,c}}{\frac{\textbf{x}_r^T \Bar{\textbf{x}}_r}{{\left\| \textbf{x}_r \right\|}{\left\| \Bar{\textbf{x}}_r \right\|}}},
$$
$$
\xi_r = \prod\limits_{\textbf{x}=\textbf{a,b,c}}{\left\| \textbf{x}_r \right\|}, \Bar{\xi_r} = \prod\limits_{\textbf{x}=\textbf{a,b,c}}{\left\| \Bar{\textbf{x}}_r \right\|}
$$
where $\Bar{\mathcal{X}} = [\![\Bar{\textbf{A}},\Bar{\textbf{B}},\Bar{\textbf{C}}]\!]$ is the estimated factors and $\mathcal{X} = [\![\textbf{A},\textbf{B},\textbf{C}]\!]$ is the true factors. $\textbf{x}_r$ is the $r^{th}$ column of factor matrices.

We treat the non-private version \methodName~factors as the benchmark for \methodName~factors. Figure \ref{fig:nicu_auc}(b) shows how the FMS changes with an increase of the privacy budget.
As the privacy budget becomes larger, the FMS increases accordingly and will gradually approximate 1, which means the discovered phenotypes between the two methods are equivalent.
This result indicates that when a stricter privacy constraint is enforced, it may negatively impact the quality of the phenotypes. Thus, there is a practical need to balance the trade-off between privacy and phenotype quality.

Table \ref{table:dp} presents a comparison between the top 1 (highest factor weight $\lambda_r$) phenotype \methodName-derived phenotype and the closest phenotype derived by its non-private version. We observe that \methodName~contains several additional noisy procedure and diagnosis elements than the non-private \methodName~version.
These extra elements are the results of adding noise to the feature factor matrices.
This is also supported in Table \ref{table:auc} as the non-private \methodName~has better predictive performance than DPFact.
Thus, the output perturbation process may interfere with the interpretability and meaningfulness of the derived phenotype.
However, there is still some utility from the \methodName-derived phenotype as experts can still distinguish this phenotype to be a heart failure phenotype.
Therefore, \methodName~still retains the ability to perform phenotype discovery.

\section{Related Work}
\subsection{Tensor Factorization}
Tensor analysis is an active research topic and has been widely applied to healthcare data \cite{kim2017federated,wang2015rubik, ho2014marble}, especially for computational phenotyping.
Moreover, several algorithms have been developed to scale tensor factorization. GigaTensor \cite{kang2012gigatensor} used MapReduce for large scale CP tensor decomposition that exploits the sparseness of the real world tensors. DFacTo \cite{choi2014dfacto} improves GigaTensor by exploring properties related to the Khatri-Rao Product and achieves faster computation time and better scalability. 
FlexiFaCT \cite{beutel2014flexifact} is a scalable MapReduce algorithm for coupled matrix-tensor decomposition using stochastic gradient descent (SGD). ADMM has also been proved to be an efficient algorithm for distributed tensor factorization \cite{kim2017federated}. However, the above proposed algorithms have the same potential limitation: the distributed data exhibits the same pattern at different local sites. That means each local tensor can be treated as a random sample from the global tensor. Thus, the algorithms are unable to model the scenario where the distribution pattern may be different at each sites. This is common in healthcare as different units (or clinics and hospitals) will have different patient populations, and may not exhibit all the computational phenotypes.

\subsection{Differential Private Factorization}
Differential privacy is widely applied to machine learning areas, especially matrix/tensor factorization, as well as on different distributed optimization frameworks and deep learning problems. Regarding tensor decomposition, there are four ways to enforce differential privacy: input perturbation, output perturbation, objective perturbation and the gradient perturbation. \cite{hua2015differentially} proposed an objective perturbation method for matrix factorization in recommendation systems. \cite{liu2015fast} proposed a new idea that sampling from the posterior distribution of a Bayesian model can sufficiently guarantee differential privacy. \cite{berlioz2015applying} compared the four different perturbation method on matrix factorization and drew the conclusion that input perturbation is the most efficient method that has the least privacy loss on recommendation systems. \cite{wang2016online} is the first proposed differentially private tensor decomposition work. It proposed a noise calibrated tensor power method. Our goal in this paper is to develop a distributed framework where data is stored at different sources, and try to preserve the privacy during knowledge transfer. Nevertheless, these works are based on a centralized framework. \cite{kim2017federated} developed a federated tensor factorization framework, but it simply preserves privacy by avoiding direct patient information sharing, rather than by applying rigorous differential privacy techniques.

\section{Conclusion}
\methodName~is a distributed large-scale tensor decomposition method that enforces differential privacy.
It is well-suited for computational phenotype from multiple sites as well as other collaborative healthcare analysis with multi-way data.
\methodName~allows data to be stored at different sites without requiring a single centralized location to perform the computation.
Moreover, our model recognizes that the learned global latent factors need not be present at all sites, allowing the discovery of both shared and site-specific computational phenotypes.
Furthermore, by adopting a communication-efficient EASGD algorithm, \methodName~greatly reduces the communication overhead.
\methodName~also successfully tackles the privacy issue under the distributed setting with limited privacy loss by the application of zCDP and parallel composition theorem.
Experiments on real-world and synthetic datasets demonstrate that our model outperforms other state-of-the-art methods in terms of communication cost, accuracy, and phenotype discovery ability.
Future work will focus on the asynchronization of the collaborative tensor factorization framework to further optimize the computation efficiency. 

\begin{acks}
This work was supported by the National Science Foundation,
award IIS-\#1838200, National Institute of Health (NIH) under award number R01GM114612, R01GM118609, and U01TR002062, and the National Institute of Health Georgia CTSA UL1TR002378. Dr. Xiaoqian Jiang is CPRIT Scholar in Cancer Research, and he was supported in part by the CPRIT RR180012, UT Stars award.
\end{acks}

\bibliographystyle{ACM-Reference-Format}
\bibliography{CIKM-sigconf}


\begin{thebibliography}{34}

%%% the Web address.

\ifx \showCODEN    \undefined \def \showCODEN     #1{\unskip}     \fi
\ifx \showDOI      \undefined \def \showDOI       #1{#1}\fi
\ifx \showISBNx    \undefined \def \showISBNx     #1{\unskip}     \fi
\ifx \showISBNxiii \undefined \def \showISBNxiii  #1{\unskip}     \fi
\ifx \showISSN     \undefined \def \showISSN      #1{\unskip}     \fi
\ifx \showLCCN     \undefined \def \showLCCN      #1{\unskip}     \fi
\ifx \shownote     \undefined \def \shownote      #1{#1}          \fi
\ifx \showarticletitle \undefined \def \showarticletitle #1{#1}   \fi
\ifx \showURL      \undefined \def \showURL       {\relax}        \fi
% invisible to TeX
\providecommand\bibfield[2]{#2}
\providecommand\bibinfo[2]{#2}
\providecommand\natexlab[1]{#1}
\providecommand\showeprint[2][]{arXiv:#2}

\bibitem[\protect\citeauthoryear{Bader, Kolda, et~al\mbox{.}}{Bader
  et~al\mbox{.}}{2017}]%
        {TTB_Software}
\bibfield{author}{\bibinfo{person}{Brett~W. Bader}, \bibinfo{person}{Tamara~G.
  Kolda}, {et~al\mbox{.}}} \bibinfo{year}{2017}\natexlab{}.
\newblock \bibinfo{title}{MATLAB Tensor Toolbox Version 3.0-dev}.
\newblock \bibinfo{howpublished}{Available online}.
\newblock
\urldef\tempurl%
\url{https://gitlab.com/tensors/tensor_toolbox}
\showURL{%
\tempurl}


\bibitem[\protect\citeauthoryear{Berlioz, Friedman, Kaafar, Boreli, and
  Berkovsky}{Berlioz et~al\mbox{.}}{2015}]%
        {berlioz2015applying}
\bibfield{author}{\bibinfo{person}{Arnaud Berlioz}, \bibinfo{person}{Arik
  Friedman}, \bibinfo{person}{Mohamed~Ali Kaafar}, \bibinfo{person}{Roksana
  Boreli}, {and} \bibinfo{person}{Shlomo Berkovsky}.}
  \bibinfo{year}{2015}\natexlab{}.
\newblock \showarticletitle{Applying differential privacy to matrix
  factorization}. In \bibinfo{booktitle}{\emph{Proceedings of the 9th ACM
  Conference on Recommender Systems}}. ACM, \bibinfo{pages}{107--114}.
\newblock


\bibitem[\protect\citeauthoryear{Beutel, Talukdar, Kumar, Faloutsos,
  Papalexakis, and Xing}{Beutel et~al\mbox{.}}{2014}]%
        {beutel2014flexifact}
\bibfield{author}{\bibinfo{person}{Alex Beutel}, \bibinfo{person}{Partha~Pratim
  Talukdar}, \bibinfo{person}{Abhimanu Kumar}, \bibinfo{person}{Christos
  Faloutsos}, \bibinfo{person}{Evangelos~E Papalexakis}, {and}
  \bibinfo{person}{Eric~P Xing}.} \bibinfo{year}{2014}\natexlab{}.
\newblock \showarticletitle{Flexifact: Scalable flexible factorization of
  coupled tensors on hadoop}. In \bibinfo{booktitle}{\emph{Proceedings of the
  2014 SDM}}. \bibinfo{pages}{109--117}.
\newblock


\bibitem[\protect\citeauthoryear{Bun and Steinke}{Bun and Steinke}{2016}]%
        {bun2016concentrated}
\bibfield{author}{\bibinfo{person}{Mark Bun} {and} \bibinfo{person}{Thomas
  Steinke}.} \bibinfo{year}{2016}\natexlab{}.
\newblock \showarticletitle{Concentrated differential privacy: Simplifications,
  extensions, and lower bounds}. In \bibinfo{booktitle}{\emph{Theory of
  Cryptography Conference}}. Springer, \bibinfo{pages}{635--658}.
\newblock


\bibitem[\protect\citeauthoryear{Chi and Kolda}{Chi and Kolda}{2012}]%
        {chi2012tensors}
\bibfield{author}{\bibinfo{person}{Eric~C Chi} {and} \bibinfo{person}{Tamara~G
  Kolda}.} \bibinfo{year}{2012}\natexlab{}.
\newblock \showarticletitle{On tensors, sparsity, and nonnegative
  factorizations}.
\newblock \bibinfo{journal}{\emph{SIAM J. Matrix Anal. Appl.}}
  \bibinfo{volume}{33}, \bibinfo{number}{4} (\bibinfo{year}{2012}),
  \bibinfo{pages}{1272--1299}.
\newblock


\bibitem[\protect\citeauthoryear{Choi and Vishwanathan}{Choi and
  Vishwanathan}{2014}]%
        {choi2014dfacto}
\bibfield{author}{\bibinfo{person}{Joon~Hee Choi} {and} \bibinfo{person}{S
  Vishwanathan}.} \bibinfo{year}{2014}\natexlab{}.
\newblock \showarticletitle{DFacTo: Distributed factorization of tensors}. In
  \bibinfo{booktitle}{\emph{NIPS}}. \bibinfo{pages}{1296--1304}.
\newblock


\bibitem[\protect\citeauthoryear{Combettes and Pesquet}{Combettes and
  Pesquet}{2011}]%
        {combettes2011proximal}
\bibfield{author}{\bibinfo{person}{Patrick~L Combettes} {and}
  \bibinfo{person}{Jean-Christophe Pesquet}.} \bibinfo{year}{2011}\natexlab{}.
\newblock \showarticletitle{Proximal splitting methods in signal processing}.
\newblock In \bibinfo{booktitle}{\emph{Fixed-point algorithms for inverse
  problems in science and engineering}}. \bibinfo{publisher}{Springer},
  \bibinfo{pages}{185--212}.
\newblock


\bibitem[\protect\citeauthoryear{Dwork, Roth, et~al\mbox{.}}{Dwork
  et~al\mbox{.}}{2014}]%
        {dwork2014algorithmic}
\bibfield{author}{\bibinfo{person}{Cynthia Dwork}, \bibinfo{person}{Aaron
  Roth}, {et~al\mbox{.}}} \bibinfo{year}{2014}\natexlab{}.
\newblock \showarticletitle{The algorithmic foundations of differential
  privacy}.
\newblock \bibinfo{journal}{\emph{Foundations and Trends{\textregistered} in
  Theoretical Computer Science}} \bibinfo{volume}{9}, \bibinfo{number}{3--4}
  (\bibinfo{year}{2014}), \bibinfo{pages}{211--407}.
\newblock


\bibitem[\protect\citeauthoryear{Dwork and Rothblum}{Dwork and
  Rothblum}{2016}]%
        {dwork2016concentrated}
\bibfield{author}{\bibinfo{person}{Cynthia Dwork} {and} \bibinfo{person}{Guy~N
  Rothblum}.} \bibinfo{year}{2016}\natexlab{}.
\newblock \showarticletitle{Concentrated differential privacy}.
\newblock \bibinfo{journal}{\emph{arXiv preprint arXiv:1603.01887}}
  (\bibinfo{year}{2016}).
\newblock


\bibitem[\protect\citeauthoryear{Dwork, Rothblum, and Vadhan}{Dwork
  et~al\mbox{.}}{2010}]%
        {dwork2010boosting}
\bibfield{author}{\bibinfo{person}{Cynthia Dwork}, \bibinfo{person}{Guy~N
  Rothblum}, {and} \bibinfo{person}{Salil Vadhan}.}
  \bibinfo{year}{2010}\natexlab{}.
\newblock \showarticletitle{Boosting and differential privacy}. In
  \bibinfo{booktitle}{\emph{2010 IEEE 51st Annual Symposium on Foundations of
  Computer Science}}. IEEE, \bibinfo{pages}{51--60}.
\newblock


\bibitem[\protect\citeauthoryear{Fredrikson, Jha, and Ristenpart}{Fredrikson
  et~al\mbox{.}}{2015}]%
        {fredrikson2015model}
\bibfield{author}{\bibinfo{person}{Matt Fredrikson}, \bibinfo{person}{Somesh
  Jha}, {and} \bibinfo{person}{Thomas Ristenpart}.}
  \bibinfo{year}{2015}\natexlab{}.
\newblock \showarticletitle{Model inversion attacks that exploit confidence
  information and basic countermeasures}. In
  \bibinfo{booktitle}{\emph{Proceedings of the 22nd ACM SIGSAC Conference on
  Computer and Communications Security}}. ACM, \bibinfo{pages}{1322--1333}.
\newblock


\bibitem[\protect\citeauthoryear{Greenhalgh, Hinder, Stramer, Bratan, and
  Russell}{Greenhalgh et~al\mbox{.}}{2010}]%
        {greenhalgh2010adoption}
\bibfield{author}{\bibinfo{person}{Trisha Greenhalgh}, \bibinfo{person}{Susan
  Hinder}, \bibinfo{person}{Katja Stramer}, \bibinfo{person}{Tanja Bratan},
  {and} \bibinfo{person}{Jill Russell}.} \bibinfo{year}{2010}\natexlab{}.
\newblock \showarticletitle{Adoption, non-adoption, and abandonment of a
  personal electronic health record: case study of HealthSpace}.
\newblock \bibinfo{journal}{\emph{Bmj}}  \bibinfo{volume}{341}
  (\bibinfo{year}{2010}), \bibinfo{pages}{c5814}.
\newblock


\bibitem[\protect\citeauthoryear{Guo and Xue}{Guo and Xue}{2013}]%
        {guo2013probabilistic}
\bibfield{author}{\bibinfo{person}{Yuhong Guo} {and} \bibinfo{person}{Wei
  Xue}.} \bibinfo{year}{2013}\natexlab{}.
\newblock \showarticletitle{Probabilistic Multi-Label Classification with
  Sparse Feature Learning.}. In \bibinfo{booktitle}{\emph{IJCAI}}.
  \bibinfo{pages}{1373--1379}.
\newblock


\bibitem[\protect\citeauthoryear{Hitaj, Ateniese, and Perez-Cruz}{Hitaj
  et~al\mbox{.}}{2017}]%
        {hitaj2017deep}
\bibfield{author}{\bibinfo{person}{Briland Hitaj}, \bibinfo{person}{Giuseppe
  Ateniese}, {and} \bibinfo{person}{Fernando Perez-Cruz}.}
  \bibinfo{year}{2017}\natexlab{}.
\newblock \showarticletitle{Deep models under the GAN: information leakage from
  collaborative deep learning}. In \bibinfo{booktitle}{\emph{Proceedings of the
  2017 ACM SIGSAC Conference on Computer and Communications Security}}. ACM,
  \bibinfo{pages}{603--618}.
\newblock


\bibitem[\protect\citeauthoryear{Ho, Ghosh, and Sun}{Ho et~al\mbox{.}}{2014}]%
        {ho2014marble}
\bibfield{author}{\bibinfo{person}{Joyce~C Ho}, \bibinfo{person}{Joydeep
  Ghosh}, {and} \bibinfo{person}{Jimeng Sun}.} \bibinfo{year}{2014}\natexlab{}.
\newblock \showarticletitle{Marble: high-throughput phenotyping from electronic
  health records via sparse nonnegative tensor factorization}. In
  \bibinfo{booktitle}{\emph{Proceedings of the 20th ACM SIGKDD}}. ACM,
  \bibinfo{pages}{115--124}.
\newblock


\bibitem[\protect\citeauthoryear{Hua, Xia, and Zhong}{Hua
  et~al\mbox{.}}{2015}]%
        {hua2015differentially}
\bibfield{author}{\bibinfo{person}{Jingyu Hua}, \bibinfo{person}{Chang Xia},
  {and} \bibinfo{person}{Sheng Zhong}.} \bibinfo{year}{2015}\natexlab{}.
\newblock \showarticletitle{Differentially Private Matrix Factorization.}. In
  \bibinfo{booktitle}{\emph{IJCAI}}. \bibinfo{pages}{1763--1770}.
\newblock


\bibitem[\protect\citeauthoryear{Johnson, Pollard, Shen, Li-wei, Feng,
  Ghassemi, Moody, Szolovits, Celi, and Mark}{Johnson et~al\mbox{.}}{2016}]%
        {johnson2016mimic}
\bibfield{author}{\bibinfo{person}{Alistair~EW Johnson}, \bibinfo{person}{Tom~J
  Pollard}, \bibinfo{person}{Lu Shen}, \bibinfo{person}{H~Lehman Li-wei},
  \bibinfo{person}{Mengling Feng}, \bibinfo{person}{Mohammad Ghassemi},
  \bibinfo{person}{Benjamin Moody}, \bibinfo{person}{Peter Szolovits},
  \bibinfo{person}{Leo~Anthony Celi}, {and} \bibinfo{person}{Roger~G Mark}.}
  \bibinfo{year}{2016}\natexlab{}.
\newblock \showarticletitle{MIMIC-III, a freely accessible critical care
  database}.
\newblock \bibinfo{journal}{\emph{Scientific data}}  \bibinfo{volume}{3}
  (\bibinfo{year}{2016}), \bibinfo{pages}{160035}.
\newblock


\bibitem[\protect\citeauthoryear{Kang, Papalexakis, Harpale, and
  Faloutsos}{Kang et~al\mbox{.}}{2012}]%
        {kang2012gigatensor}
\bibfield{author}{\bibinfo{person}{U Kang}, \bibinfo{person}{Evangelos
  Papalexakis}, \bibinfo{person}{Abhay Harpale}, {and}
  \bibinfo{person}{Christos Faloutsos}.} \bibinfo{year}{2012}\natexlab{}.
\newblock \showarticletitle{Gigatensor: scaling tensor analysis up by 100
  times-algorithms and discoveries}. In \bibinfo{booktitle}{\emph{Proceedings
  of the 18th ACM SIGKDD}}. ACM, \bibinfo{pages}{316--324}.
\newblock


\bibitem[\protect\citeauthoryear{Kim, El-Kareh, Sun, Yu, and Jiang}{Kim
  et~al\mbox{.}}{2017a}]%
        {kim2017discriminative}
\bibfield{author}{\bibinfo{person}{Yejin Kim}, \bibinfo{person}{Robert
  El-Kareh}, \bibinfo{person}{Jimeng Sun}, \bibinfo{person}{Hwanjo Yu}, {and}
  \bibinfo{person}{Xiaoqian Jiang}.} \bibinfo{year}{2017}\natexlab{a}.
\newblock \showarticletitle{Discriminative and distinct phenotyping by
  constrained tensor factorization}.
\newblock \bibinfo{journal}{\emph{Scientific reports}} \bibinfo{volume}{7},
  \bibinfo{number}{1} (\bibinfo{year}{2017}), \bibinfo{pages}{1114}.
\newblock


\bibitem[\protect\citeauthoryear{Kim, Sun, Yu, and Jiang}{Kim
  et~al\mbox{.}}{2017b}]%
        {kim2017federated}
\bibfield{author}{\bibinfo{person}{Yejin Kim}, \bibinfo{person}{Jimeng Sun},
  \bibinfo{person}{Hwanjo Yu}, {and} \bibinfo{person}{Xiaoqian Jiang}.}
  \bibinfo{year}{2017}\natexlab{b}.
\newblock \showarticletitle{Federated tensor factorization for computational
  phenotyping}. In \bibinfo{booktitle}{\emph{Proceedings of the 23rd ACM
  SIGKDD}}. ACM, \bibinfo{pages}{887--895}.
\newblock


\bibitem[\protect\citeauthoryear{Liu, Ji, and Ye}{Liu et~al\mbox{.}}{2009}]%
        {liu2009multi}
\bibfield{author}{\bibinfo{person}{Jun Liu}, \bibinfo{person}{Shuiwang Ji},
  {and} \bibinfo{person}{Jieping Ye}.} \bibinfo{year}{2009}\natexlab{}.
\newblock \showarticletitle{Multi-task feature learning via efficient $\ell_{2,
  1}$-norm minimization}. In \bibinfo{booktitle}{\emph{UAI}}.
  \bibinfo{pages}{339--348}.
\newblock


\bibitem[\protect\citeauthoryear{Liu, Wang, and Smola}{Liu
  et~al\mbox{.}}{2015}]%
        {liu2015fast}
\bibfield{author}{\bibinfo{person}{Ziqi Liu}, \bibinfo{person}{Yu-Xiang Wang},
  {and} \bibinfo{person}{Alexander Smola}.} \bibinfo{year}{2015}\natexlab{}.
\newblock \showarticletitle{Fast differentially private matrix factorization}.
  In \bibinfo{booktitle}{\emph{Proceedings of the 9th ACM RecSys}}.
  \bibinfo{pages}{171--178}.
\newblock


\bibitem[\protect\citeauthoryear{Nie, Huang, Cai, and Ding}{Nie
  et~al\mbox{.}}{2010}]%
        {nie2010efficient}
\bibfield{author}{\bibinfo{person}{Feiping Nie}, \bibinfo{person}{Heng Huang},
  \bibinfo{person}{Xiao Cai}, {and} \bibinfo{person}{Chris~H Ding}.}
  \bibinfo{year}{2010}\natexlab{}.
\newblock \showarticletitle{Efficient and robust feature selection via joint
  l2, 1-norms minimization}. In \bibinfo{booktitle}{\emph{NeurIPS}}.
  \bibinfo{pages}{1813--1821}.
\newblock


\bibitem[\protect\citeauthoryear{Richesson, Sun, Pathak, Kho, and
  Denny}{Richesson et~al\mbox{.}}{2016}]%
        {richesson2016clinical}
\bibfield{author}{\bibinfo{person}{Rachel~L Richesson}, \bibinfo{person}{Jimeng
  Sun}, \bibinfo{person}{Jyotishman Pathak}, \bibinfo{person}{Abel~N Kho},
  {and} \bibinfo{person}{Joshua~C Denny}.} \bibinfo{year}{2016}\natexlab{}.
\newblock \showarticletitle{Clinical phenotyping in selected national networks:
  demonstrating the need for high-throughput, portable, and computational
  methods}.
\newblock \bibinfo{journal}{\emph{Artificial intelligence in medicine}}
  \bibinfo{volume}{71} (\bibinfo{year}{2016}), \bibinfo{pages}{57--61}.
\newblock


\bibitem[\protect\citeauthoryear{Shokri, Stronati, Song, and Shmatikov}{Shokri
  et~al\mbox{.}}{2017}]%
        {shokri2017membership}
\bibfield{author}{\bibinfo{person}{Reza Shokri}, \bibinfo{person}{Marco
  Stronati}, \bibinfo{person}{Congzheng Song}, {and} \bibinfo{person}{Vitaly
  Shmatikov}.} \bibinfo{year}{2017}\natexlab{}.
\newblock \showarticletitle{Membership inference attacks against machine
  learning models}. In \bibinfo{booktitle}{\emph{Security and Privacy (SP),
  2017 IEEE Symposium on}}. IEEE, \bibinfo{pages}{3--18}.
\newblock


\bibitem[\protect\citeauthoryear{Wang and Anandkumar}{Wang and
  Anandkumar}{2016}]%
        {wang2016online}
\bibfield{author}{\bibinfo{person}{Yining Wang} {and} \bibinfo{person}{Anima
  Anandkumar}.} \bibinfo{year}{2016}\natexlab{}.
\newblock \showarticletitle{Online and differentially-private tensor
  decomposition}. In \bibinfo{booktitle}{\emph{NeurIPS}}.
  \bibinfo{pages}{3531--3539}.
\newblock


\bibitem[\protect\citeauthoryear{Wang, Chen, Ghosh, Denny, Kho, Chen, Malin,
  and Sun}{Wang et~al\mbox{.}}{2015}]%
        {wang2015rubik}
\bibfield{author}{\bibinfo{person}{Yichen Wang}, \bibinfo{person}{Robert Chen},
  \bibinfo{person}{Joydeep Ghosh}, \bibinfo{person}{Joshua~C Denny},
  \bibinfo{person}{Abel Kho}, \bibinfo{person}{You Chen},
  \bibinfo{person}{Bradley~A Malin}, {and} \bibinfo{person}{Jimeng Sun}.}
  \bibinfo{year}{2015}\natexlab{}.
\newblock \showarticletitle{Rubik: Knowledge guided tensor factorization and
  completion for health data analytics}. In
  \bibinfo{booktitle}{\emph{Proceedings of the 21th ACM SIGKDD}}. ACM,
  \bibinfo{pages}{1265--1274}.
\newblock


\bibitem[\protect\citeauthoryear{Wei and Denny}{Wei and Denny}{2015}]%
        {wei2015extracting}
\bibfield{author}{\bibinfo{person}{Wei-Qi Wei} {and} \bibinfo{person}{Joshua~C
  Denny}.} \bibinfo{year}{2015}\natexlab{}.
\newblock \showarticletitle{Extracting research-quality phenotypes from
  electronic health records to support precision medicine}.
\newblock \bibinfo{journal}{\emph{Genome medicine}} \bibinfo{volume}{7},
  \bibinfo{number}{1} (\bibinfo{year}{2015}), \bibinfo{pages}{41}.
\newblock


\bibitem[\protect\citeauthoryear{Wu, Li, Kumar, Chaudhuri, Jha, and
  Naughton}{Wu et~al\mbox{.}}{2017}]%
        {wu2017bolt}
\bibfield{author}{\bibinfo{person}{Xi Wu}, \bibinfo{person}{Fengan Li},
  \bibinfo{person}{Arun Kumar}, \bibinfo{person}{Kamalika Chaudhuri},
  \bibinfo{person}{Somesh Jha}, {and} \bibinfo{person}{Jeffrey Naughton}.}
  \bibinfo{year}{2017}\natexlab{}.
\newblock \showarticletitle{Bolt-on differential privacy for scalable
  stochastic gradient descent-based analytics}. In
  \bibinfo{booktitle}{\emph{Proceedings of the 2017 ACM International
  Conference on Management of Data}}. ACM, \bibinfo{pages}{1307--1322}.
\newblock


\bibitem[\protect\citeauthoryear{Xu, Li, Lin, Normand, Lagu, Desai, Duan,
  Kroch, and Krumholz}{Xu et~al\mbox{.}}{2016}]%
        {xu2016hospital}
\bibfield{author}{\bibinfo{person}{Xiao Xu}, \bibinfo{person}{Shu-Xia Li},
  \bibinfo{person}{Haiqun Lin}, \bibinfo{person}{SL Normand},
  \bibinfo{person}{Tara Lagu}, \bibinfo{person}{Nihar Desai},
  \bibinfo{person}{Michael Duan}, \bibinfo{person}{Eugene~A Kroch}, {and}
  \bibinfo{person}{Harlan~M Krumholz}.} \bibinfo{year}{2016}\natexlab{}.
\newblock \showarticletitle{Hospital Phenotypes in the Management of Patients
  Admitted for Acute Myocardial Infarction.}
\newblock \bibinfo{journal}{\emph{Medical care}} \bibinfo{volume}{54},
  \bibinfo{number}{10} (\bibinfo{year}{2016}), \bibinfo{pages}{929--936}.
\newblock


\bibitem[\protect\citeauthoryear{Yang, Shen, Ma, Huang, and Zhou}{Yang
  et~al\mbox{.}}{2011}]%
        {yang2011l2}
\bibfield{author}{\bibinfo{person}{Yi Yang}, \bibinfo{person}{Heng~Tao Shen},
  \bibinfo{person}{Zhigang Ma}, \bibinfo{person}{Zi Huang}, {and}
  \bibinfo{person}{Xiaofang Zhou}.} \bibinfo{year}{2011}\natexlab{}.
\newblock \showarticletitle{l2, 1-norm regularized discriminative feature
  selection for unsupervised learning}. In \bibinfo{booktitle}{\emph{IJCAI}},
  Vol.~\bibinfo{volume}{22}. \bibinfo{pages}{1589}.
\newblock


\bibitem[\protect\citeauthoryear{Youssef, Abd-Elrahman, Shehab, Abd-Elrheem,
  et~al\mbox{.}}{Youssef et~al\mbox{.}}{2015}]%
        {youssef2015incidence}
\bibfield{author}{\bibinfo{person}{Doaa Youssef}, \bibinfo{person}{Hadeel
  Abd-Elrahman}, \bibinfo{person}{Mohamed~M Shehab}, \bibinfo{person}{Mohamed
  Abd-Elrheem}, {et~al\mbox{.}}} \bibinfo{year}{2015}\natexlab{}.
\newblock \showarticletitle{Incidence of acute kidney injury in the neonatal
  intensive care unit}.
\newblock \bibinfo{journal}{\emph{Saudi journal of kidney diseases and
  transplantation}} \bibinfo{volume}{26}, \bibinfo{number}{1}
  (\bibinfo{year}{2015}), \bibinfo{pages}{67}.
\newblock


\bibitem[\protect\citeauthoryear{Yu, Liu, Pu, Gursoy, and Truex}{Yu
  et~al\mbox{.}}{2019}]%
        {yu2019differentially}
\bibfield{author}{\bibinfo{person}{Lei Yu}, \bibinfo{person}{Ling Liu},
  \bibinfo{person}{Calton Pu}, \bibinfo{person}{Mehmet~Emre Gursoy}, {and}
  \bibinfo{person}{Stacey Truex}.} \bibinfo{year}{2019}\natexlab{}.
\newblock \showarticletitle{Differentially Private Model Publishing for Deep
  Learning}.
\newblock \bibinfo{journal}{\emph{arXiv preprint arXiv:1904.02200}}
  (\bibinfo{year}{2019}).
\newblock


\bibitem[\protect\citeauthoryear{Zhang, Choromanska, and LeCun}{Zhang
  et~al\mbox{.}}{2015}]%
        {zhang2015deep}
\bibfield{author}{\bibinfo{person}{Sixin Zhang}, \bibinfo{person}{Anna~E
  Choromanska}, {and} \bibinfo{person}{Yann LeCun}.}
  \bibinfo{year}{2015}\natexlab{}.
\newblock \showarticletitle{Deep learning with elastic averaging SGD}. In
  \bibinfo{booktitle}{\emph{NeurIPS}}. \bibinfo{pages}{685--693}.
\newblock


\end{thebibliography}

\clearpage
\appendix
\section{Appendix}
\label{sec:appendix}

This section provides supplementary information.

\subsection{$Lipschitz$ Constant}
Below is the calculation of the $Lipschitz$ constant in objective function \eqref{eq:obj_feature}, which will be used for calculating $\beta$ in section \ref{sec:priv}. Since $\textbf{B}^{[t]}$ and $\textbf{C}^{[t]}$ play the same role in the objective function in \eqref{eq:obj_feature}, we take $\textbf{B}^{[t]}$ as an example.
$f(\textbf{B}^{[t]})$ can be rewritten as
\begin{equation}
f(\mathbf{B}^{[t]})=\underbrace{{\frac{1}{2}}{\left\| O^{[t]}_{(n)}-\mathbf{B}^{[t]}\bm{\Pi}_B \right\|}_F^2}_{\mathcal{F}_B}+\underbrace{{\frac{\gamma}{2}}{\left\| \mathbf{B}^{[t]} - \mathbf{B}\right\|}_F^2}_{\mathcal{Q}_B},
\label{eq:matricized_obj}
\end{equation}
where $\mathcal{O}_{(n)}^{[t]}$ is the matricization of the local observed tensor $\mathcal{O}^{[t]}$, $\bm{\Pi}_B = \textbf{C}^{[t]}\odot \textbf{A}^{[t]}$. Thus $f(\textbf{B}^{[t]})$ is the combination of $\mathcal{F}_B$ and $\mathcal{Q}_B$. We provide analysis of $Lipschitz$ continuity of \eqref{eq:matricized_obj} separately for $\mathcal{F}_B$ and $\mathcal{Q}_B$. The analysis for $\mathcal{F}_B$ could also be adopted into equation \eqref{eq:patnorm} in section \ref{sec:pfm}.

The gradient of $\mathcal{F}_B({\textbf{B}}^{[t]})$ is calculated as
$\nabla\mathcal{F}_B({\textbf{B}}^{[t]})=-O^{[t]}_{(n)}\bm{\Pi}_B+{\textbf{B}^{[t]}} \bm{\Gamma}_B$, 
where $\bm{\Gamma}_B=(\textbf{A}^{[t]\top}\textbf{A}^{[t]})*(\textbf{C}^{[t]\top}\textbf{C}^{[t]})$. Furthermore, for any ${\textbf{B}}^{[t]}_1, {\textbf{B}}^{[t]}_2 \in R_{+}^{j\times n}$, we have
\begin{equation}
\begin{aligned}
{\left\| \nabla\mathcal{F}_B({\textbf{B}}^{[t]}_1)-\nabla\mathcal{F}_B({\textbf{B}}^{[t]}_2) \right\|}_F 
& = {\left\| {\textbf{B}^{[t]}_1}\bm{\Gamma}_B-{\textbf{B}^{[t]}}_2 \bm{\Gamma}_B\right\|}_F \\
& = {\left\| ({\textbf{B}^{[t]}_1}-{\textbf{B}^{[t]}_2}) \bm{\Gamma}_B\right\|}_F \\
& \leq {\left\| \bm{\Gamma}_B \right\|}_F{\left\| {\mathbf{B}^{[t]}_1}-{\mathbf{B}^{[t]}_2} \right\|}_F. \\
\end{aligned}
\label{eq:fb}
\end{equation}

The gradient of $\mathcal{Q}_B$ is calculated as
$\nabla\mathcal{Q}_B({\textbf{B}}^{[t]})=\gamma \textbf{B}^{[t]}$. Similar to $\nabla\mathcal{F}_B({\textbf{B}}^{[t]})$, for any ${\textbf{B}}^{[t]}_1, {\textbf{B}}^{[t]}_2 \in R_{+}^{j\times n}$, we have
\begin{equation}
\begin{aligned}
{\left\| \nabla\mathcal{F}_B({\textbf{B}}^{[t]}_1)-\nabla\mathcal{F}_B({\textbf{B}}^{[t]}_2) \right\|}_F
& = {\left\| \gamma \textbf{B}^{[t]}_1 - \gamma \textbf{B}^{[t]}_2 \right\|}_F \\
& = {\left\| \gamma I_j (\textbf{B}^{[t]}_1 - \textbf{B}^{[t]}_2)\right\|}_F\\
& \leq {\left\| \gamma I_j \right\|}_F {\left\| \textbf{B}^{[t]}_1 - \textbf{B}^{[t]}_2 \right\|}_F
\end{aligned}
\label{eq:qb}
\end{equation}

By combining the results in \eqref{eq:fb} and \eqref{eq:qb}, we thus get the $Lipschitz$ constant of $\nabla f(\textbf{B}^{[t]})$ as the Frobenius norm of
$$
(\textbf{A}^{[t]\top}\textbf{A}^{[t]})*(\textbf{C}^{[t]\top}\textbf{C}^{[t]})+\gamma I_j.
$$

\subsection{$L_{2,1}$ Regularization Parameter}
Figure \ref{fig:icu_norms} illustrates the effect of choosing different value of $\mu$ on the column norm of the patient matrix for each ICUs in MIMIC-III dataset. We observe that smaller $\mu$ has minimal effect on the column sparsity, as there are no columns that are set to 0. However, if we set $\mu$ to be too high (i.e., $\mu = \{5, 6, 8, 6, 5, 0.9\}$ for each ICU respectively), then it ``turns off" a large portion of the factors and prevents \methodName~from generating useful phenotypes.
Based on the figure, we choose $\mu = \{1, 1.8, 3.2, 1.8, 1.5, 0.6\}$ for TSICU, SICU, MICU, CSRU, CCU, NICU respectively, as the optimal solution for MIMIC-III as there are still noticeable differences in the column magnitude (i.e, the phenotypes have a natural ordering within each location) but also provides flexibility to have at least one unshared column (see component 2 and 4).
\begin{figure*}[hbp]
\setlength{\abovecaptionskip}{0cm}
\setlength{\belowcaptionskip}{0cm}
\centering
\includegraphics[width=6.75in,trim={0.5mm 0.5mm 0.5mm 0.5mm},clip]{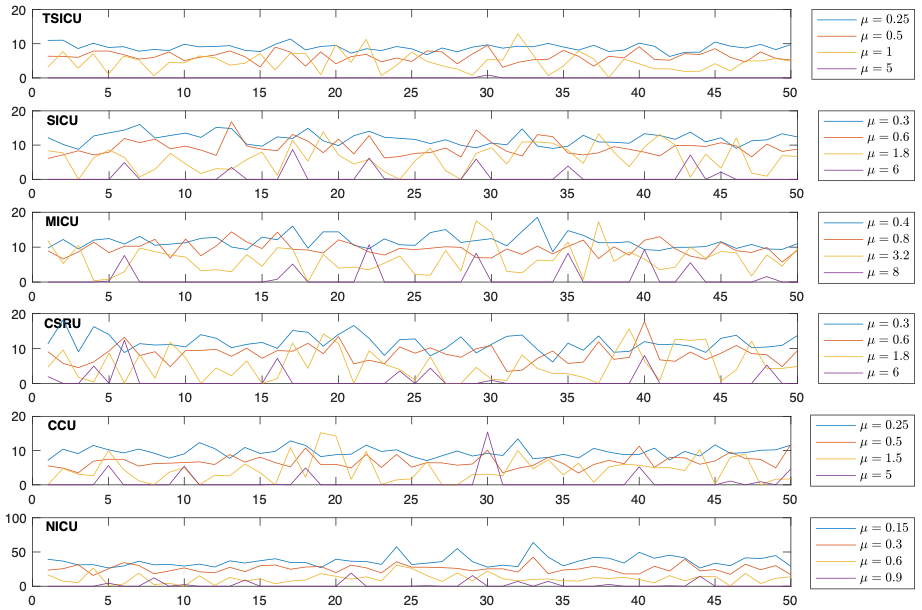}
\caption{Norm of each ICU with different regularization term $\mu$ \\(y-axis is the norm, x-axis is the rank)}
\label{fig:icu_norms}
\end{figure*}

\subsection{Phenotype Selection}
Table \ref{table:overall_stats} provides information in supplementary for phenotype selection of the overall pattern. Similar to table \ref{table:nicu_stats}, among the 50 phenotypes generated by \methodName, we selected 20 phenotypes that are statistically significant for mortality prediction. We reported 6 representative phenotypes which has the highest factor weights ($\lambda$).

\begin{table}
  \setlength{\abovecaptionskip}{0cm}
  \setlength{\belowcaptionskip}{0cm}
  \centering
  \begin{tabular}{ p{2.6cm} r r r r}
  \toprule
    \textbf{Phenotypes} & \textbf{Coef} & \textbf{p-value} & \textbf{$\lambda$} & \textbf{Prevalence}\\
    \midrule
    2 & 1.46 & <0.001 & 163 & 21.67\\
    3: Hypertension  & -1.53 & <0.001 & 249 & 20.57\\
    6 & -1.19 & <0.001 & 145 & 23.53\\
    7 & -2.34 & <0.001 & 49 & 18.17\\
    8 & 2.89 & <0.001 & 162 & 24.31\\
    9 & 2.88 & <0.001 & 116 & 22.81\\
    14 & -1.72 & <0.001 & 94 & 18.25\\
    16 & 2.29 & <0.001 & 79 & 17.48\\
    17  & 4.21 & <0.001 & 166 & 21.96\\
    18 & 3.66 & <0.001 & 69 & 16.79\\
    20 & -1.51 & <0.001 & 95 & 20.72\\
    22 & 2.17 & <0.001 & 137 & 19.19\\
    25: Heart failure & -6.56 & <0.001 & 278 & 16.40\\
    30: Acute kidney injury & 0.46 & <0.001 & 203 & 26.35\\
    32 & 1.65 & <0.001 & 109 & 18.73\\
    37 & 1.42 & <0.001 & 116 & 22.50\\
    42: Gastritis and gastroduodenitis & -2.24 & <0.001 & 187 & 22.88\\
    47: Cardiac surgery & -2.94 & <0.001 & 223 & 16.55\\
    49 & 2.91 & <0.001 & 113 & 22.73\\
    50: Chronic ischemic heart disease & 1.38 & <0.001 & 207 & 27.45\\
    \bottomrule
\end{tabular}
\caption{Logistic regression results for phenotype selection}
\label{table:overall_stats}
\end{table}

\begin{table}
  \setlength{\abovecaptionskip}{0cm}
  \setlength{\belowcaptionskip}{0cm}
  \centering
  \begin{tabular}{l r r r}
  \toprule
    \textbf{\# of Sites} & \textbf{MIMIC-III} & \textbf{CMS} & \textbf{Synthetic}\\
    \midrule
    1 & 18.73 & 22.89 & 1.55\\
    5 & 93.62 & 114.42 & 7.75\\
    10 & 189.83  & 228.83 & 15.50\\
    \bottomrule
\end{tabular}
\caption{Communication cost of \methodName~for different number of sites (Seconds)}
\label{tab:commu2}
\end{table}

\end{document}